%% file: bmvc_final.tex
\PassOptionsToPackage{table}{xcolor}
\documentclass{bmvc2k}
\usepackage{amsmath,amssymb}
\usepackage{graphicx}
\usepackage{booktabs}
\usepackage{algorithm,algorithmic}
\usepackage{hyperref}
\usepackage{multirow}

\usepackage{amsthm}

\newtheorem{lemma}{Lemma}

\newtheorem{corollary}{Corollary}

\title{Temperature-Adaptive\\ Transformed Teacher Matching}

\author{Hiroaki Aizawa\inst{1}\orcidlink{0000-0002-6241-3973} \and
Yoshikazu Hayashi\inst{2}\orcidlink{0009-0000-6251-3950}}

\addauthor{Hiroaki Aizawa}{hiroaki-aizawa@hiroshima-u.ac.jp}{1}
\addauthor{Yoshikazu Hayashi}{hayashi@cv.info.gifu-u.ac.jp}{2}

\addinstitution{
Graduate School of Advanced Science and Engineering\\ Hiroshima University, Japan\\
}
\addinstitution{
Intelligent Production Technology Research \& Development Center for Aerospace \\
Gifu University, Japan\\
}

\runninghead{H.Aizawa and Y.Hayashi}{Temperature-Adaptive TTM}

\def\ie{\emph{i.e}\bmvaOneDot}

\begin{document}

\maketitle

\begin{abstract}
Temperature scaling is a core component of knowledge distillation, yet its role and effect are still not fully understood. Transformed Teacher Matching (TTM) clarifies the role of temperature scaling by applying it only to the teacher distribution and interpreting the resulting objective as standard distillation with an implicit R\'enyi entropy regularization on the student. However, TTM still relies on a fixed temperature and does not specify how the teacher-side temperature should be adapted for individual samples. In this paper, we introduce a sample-wise inverse-temperature update for TTM by locally minimizing the Kullback-Leibler divergence between the temperature-scaled teacher distribution and the student's prediction. We derive closed-form first and second derivatives with respect to the inverse temperature, and show that they can be expressed using variance and covariance statistics of centered teacher and student logits under the transformed teacher weighting. This yields an efficient curvature-aware update that requires one softmax evaluation and a constant number of class-wise weighted sums. Experiments on standard image classification distillation benchmarks show that our temperature adaptation generally improves TTM and WTTM, while remaining competitive with or outperforming prior temperature-adaptive distillation baselines.
\end{abstract}

\section{Introduction}
\label{sec:intro}

Knowledge distillation (KD) trains a student model by transferring knowledge from a teacher model, improving accuracy without increasing inference cost~\cite{kd}. This paradigm is particularly valuable when training resources or high-quality labeled data are limited, since it can deliver strong performance with a smaller model budget. In practice, however, the effectiveness of KD depends on several design choices that can be sensitive to the teacher, student, and dataset, and these choices often require careful tuning.

\input{figs/fig_overview}

A core component of standard KD is temperature scaling, which smooths predictive distributions to expose class-to-class relations, often referred to as dark knowledge. In the common formulation~\cite{kd}, temperature scaling is applied to both teacher and student outputs, and the student is trained to match the resulting softened distributions. This smoothing assigns meaningful probability mass and gradients to non-maximum classes, making it easier for the student to learn the teacher's relative class structure. Despite its empirical success, the role of temperature scaling in KD has been less clear, and the design choice of whether to apply temperature to the teacher, the student, or both has historically been ambiguous.

Transformed Teacher Matching (TTM)~\cite{ttm} resolves this ambiguity in the distillation objective by using temperature scaling only on the teacher distribution and matching the transformed teacher distribution to the student's temperature-free prediction. Beyond this objective-level clarification, TTM provides a principled interpretation of temperature in distillation. By reformulating the teacher-side temperature transformation via power transforms, it shows that the resulting objective combines distribution matching with an implicit R\'enyi entropy regularization on the student.

Temperature scaling is also widely used beyond distillation, including calibration~\cite{calibration}, out-of-distribution detection~\cite{odin}, semi- and self-supervised learning~\cite{simclr,dino,meanteacher,mixmatch}, and learning with label noise~\cite{dividemix}. In these settings, the temperature can substantially affect optimization behavior and final performance. In practice, however, selecting an effective temperature often depends on the particular teacher, student, and dataset, and it is commonly chosen through empirical tuning. While TTM clarifies the objective by applying temperature scaling only on the teacher side, it still relies on a fixed temperature and does not specify how the teacher-side temperature should vary across samples.

In this paper, we address this limitation by introducing a per-sample temperature adaptation method for TTM. Our approach optimizes the teacher-side inverse temperature for each sample by locally minimizing the Kullback-Leibler divergence between the transformed teacher distribution and the student's prediction (Fig.~\ref{fig:overview}). We derive closed-form first and second derivatives with respect to the inverse temperature, and show that they can be written as compact variance-covariance expressions of centered teacher and student logits under the transformed teacher weighting. These derivatives lead to an efficient curvature-aware update that adds negligible overhead and can be directly integrated into TTM and WTTM. As a result, the proposed method preserves the TTM framework while adapting the teacher target, and consequently the implicit R\'enyi entropy regularization, at the sample level. Experiments on standard image classification distillation benchmarks show that our temperature adaptation generally improves TTM and WTTM, while remaining competitive with or outperforming prior temperature-adaptive distillation baselines.

\section{Related Work}
\label{sec:related_work}
Knowledge distillation (KD) transfers knowledge from a teacher model to a student model by matching their predictive distributions, improving accuracy without increasing inference cost~\cite{kd}. Prior work has explored diverse transfer signals, including feature-based distillation~\cite{fitnet,at,vid}, relation and similarity matching methods~\cite{pkt,rkd,similaritykd}, and contrastive objectives for representation distillation~\cite{crd}. On the logits side, a large line of work improves distillation by refining target distributions or loss formulations~\cite{dkd,dist}, while self-distillation provides a complementary direction~\cite{bornagain}.

Temperature scaling is central to KD and is also widely used in related settings such as calibration~\cite{calibration}, out-of-distribution detection~\cite{odin}, and semi- or self-supervised learning~\cite{simclr,dino,meanteacher,mixmatch}. Recent KD methods adapt temperature by introducing curriculum temperature schedules~\cite{ctkd}, logit standardization~\cite{ls}, entropy-based adaptation~\cite{eakd}, and precomputed sample-wise temperatures based on pretrained-teacher difficulty~\cite{yang2025atd}. In contrast to temperature adaptation, Transformed Teacher Matching (TTM) applies temperature scaling only to the teacher distribution and shows that the resulting formulation can be viewed as standard distillation with an implicit R\'enyi entropy regularization on the student~\cite{ttm}. WTTM further extends TTM with sample-wise weighting to emphasize informative teacher signals~\cite{ttm}. Our work complements TTM and WTTM by adapting the teacher-side inverse temperature at the sample level using a curvature-aware update derived from closed-form first and second derivatives of the teacher-student mismatch, enabling temperature adaptation within the TTM framework.

\section{Preliminaries}
\label{sec:preliminaries}

\subsection{Setting and Notation}
\label{ssec:notation}
We consider $K$-class classification under a knowledge distillation setting. Given an input-label pair $(\mathbf{x},y)$ with $y\in\{1,\dots,K\}$, we train a student model $f_s(\cdot;\theta_s)$ using both the ground-truth label and the teacher's predictive distribution produced by a teacher model $f_t(\cdot;\theta_t)$. The teacher parameters $\theta_t$ are fixed during distillation, while the student parameters $\theta_s$ are optimized. Let the teacher and student models produce logits $\mathbf{z}_t=f_t(\mathbf{x};\theta_t)\in\mathbb{R}^K$ and $\mathbf{z}_s=f_s(\mathbf{x};\theta_s)\in\mathbb{R}^K$. The corresponding temperature-free predictive distributions are
\begin{equation}
p=\mathrm{softmax}(\mathbf{z}_t),\qquad
q=\mathrm{softmax}(\mathbf{z}_s),
\label{eq:pq}
\end{equation}
where $\bigl(\mathrm{softmax}(\mathbf{z})\bigr)_i=\exp(z_i)\big/\sum_{j=1}^K \exp(z_j)$. When the dependence on model parameters is important, we write $p_{\theta_t}=\mathrm{softmax}(f_t(\mathbf{x};\theta_t))$ and $q_{\theta_s}=\mathrm{softmax}(f_s(\mathbf{x};\theta_s))$; otherwise, we use  $p$ and $q$ for notational simplicity.

For distillation, we introduce a temperature parameter $T>0$ to control the sharpness (or smoothness) of the predictive distributions. Specifically, we define the temperature-scaled teacher and student distributions as
\begin{equation}
p_T=\mathrm{softmax}(\mathbf{z}_t/T),\qquad
q_T=\mathrm{softmax}(\mathbf{z}_s/T).
\label{eq:pTqT}
\end{equation}
We use $\mathcal{L}_{\mathrm{CE}}(\cdot,\cdot)$ for the cross-entropy loss and $D_\mathrm{KL}(\cdot\|\cdot)$ for the Kullback-Leibler divergence.

\subsection{Knowledge Distillation and Transformed Teacher Matching}
\label{ssec:kd_and_ttm}
\paragraph{Knowledge distillation (KD)~\cite{kd}.}
KD trains the student by combining supervised learning with distribution matching between temperature-scaled teacher and student predictions. A common formulation is
\begin{equation}
\mathcal{L}_{\mathrm{KD}}
=(1-\lambda)\,\mathcal{L}_{\mathrm{CE}}(y,q)
+\lambda T^2\,D_\mathrm{KL}(p_T\|q_T),
\end{equation}
where $\lambda\in[0,1]$ balances the supervised and distillation terms.

\paragraph{Transformed teacher matching (TTM)~\cite{ttm}.}
To clarify the role of temperature on the teacher side, TTM interprets temperature scaling as a \emph{power transform}. Given the temperature-free teacher prediction $p$, we define the power-transformed distribution
\begin{equation}
p^{(\gamma)}_i=\frac{p_i^\gamma}{\sum_{j=1}^K p_j^\gamma},\qquad \gamma>0.
\label{eq:power_transform}
\end{equation}
By setting $\gamma=1/T$, we obtain the equivalence $p_T=p^{(\gamma)}$ and we compute $p^{(\gamma)}$ directly from logits as
\begin{equation}
p^{(\gamma)}=\mathrm{softmax}(\gamma\,\mathbf{z}_t).
\label{eq:p_gamma_softmax}
\end{equation}

Building on the power-transform view, TTM makes the connection to the standard KD objective explicit. In particular, with the parameter correspondence $\gamma=1/T$ and $\beta=\frac{\lambda}{1-\lambda}T$, the TTM loss can be rewritten as
\begin{equation}
\mathcal{L}_{\mathrm{TTM}}
=\mathcal{L}_{\mathrm{CE}}(y,q)
+\beta\,D_\mathrm{KL}(p^{(\gamma)}\|q)
=\frac{1}{1-\lambda}\Bigl[\mathcal{L}_{\mathrm{KD}}-\lambda T(T-1)\,\mathcal{H}_{1/T}(q)\Bigr],
\label{eq:L_ttm}
\end{equation}
where $\mathcal{H}_{\alpha}(q)$ denotes the R\'enyi entropy of order $\alpha$~\cite{renyi}. This identity highlights that TTM is not only a distribution-matching objective but also introduces an implicit R\'enyi entropy regularization on the student prediction.

\paragraph{Weighted TTM (WTTM)~\cite{ttm}.}
WTTM further extends TTM by assigning a sample-wise weight to the distillation term based on the sharpness of the teacher target. Concretely, WTTM multiplies the distillation coefficient $\beta$ in Eq.~\eqref{eq:L_ttm} by a weight $U_{\gamma}(p)$ computed for each sample:
\begin{equation}
\mathcal{L}_{\mathrm{WTTM}}
=\mathcal{L}_{\mathrm{CE}}(y,q)
+\beta U_{\gamma}(p)D_\mathrm{KL}(p_T\|q),
\quad U_{\gamma}(p)=\sum_{j=1}^{K} p_j^{\gamma},
\label{eq:L_wttm}
\end{equation}
where $U_{\gamma}(p)$ is a power-sum of the teacher distribution that increases when the teacher target is smoother, thereby upweighting samples for which dark knowledge is more informative, while keeping the teacher-side temperature unchanged.

\section{Methodology}
\label{sec:methodology}

\subsection{Overview: Per-Sample Temperature Adaptation for TTM}
\label{ssec:overview}
We address the remaining question in TTM: \emph{how to adapt the teacher temperature}. Our goal is to enable sample-adaptive temperature scaling within the TTM framework. To this end, we optimize the inverse temperature $\gamma=1/T$ for each input under fixed logits. Specifically, we form the teacher distribution $p^{(\gamma)}=\mathrm{softmax}(\gamma\,\mathbf{z}_t)$ and the student prediction $q=\mathrm{softmax}(\mathbf{z}_s)$, and minimize the per-sample TTM discrepancy $D_\mathrm{KL}\!\bigl(p^{(\gamma)}\|q\bigr)$.

\paragraph{Local sensitivity and curvature of the KL discrepancy.}
To understand how the teacher temperature affects TTM, \ie, to quantify how the teacher-side inverse temperature $\gamma$ affects TTM at the sample level, we analyze the local behavior of $D_\mathrm{KL}\!\bigl(p^{(\gamma)}\|q\bigr)$ under small perturbations $\Delta\gamma$. This discrepancy is the teacher-student matching term in TTM, and its value can vary substantially across samples as the student prediction changes during training. Therefore, its local variation provides a natural signal for how the teacher-side transformation should be adjusted for each sample. Concretely, we approximate the change in the discrepancy by a second-order Taylor expansion around the current $\gamma$:
\begin{equation}
D_\mathrm{KL}\!\bigl(p^{(\gamma+\Delta\gamma)}\|q\bigr)
\approx
\underbrace{D_\mathrm{KL}\!\bigl(p^{(\gamma)}\|q\bigr)}_{\text{constant}}
+\underbrace{\frac{d}{d\gamma}D_\mathrm{KL}\bigl(p^{(\gamma)}\|q\bigr)\Delta\gamma}_{\text{local sensitivity}}
+\underbrace{\frac{1}{2}\frac{d^2}{d\gamma^2}D_\mathrm{KL}\bigl(p^{(\gamma)}\|q\bigr)(\Delta\gamma)^2}_{\text{local curvature}},
\label{eq:taylor_KL}
\end{equation}
which characterizes both the local sensitivity (first derivative) and curvature (second derivative) with respect to $\gamma$. This quadratic model provides the local surrogate used below, and its derivatives are derived in Sec.~\ref{ssec:closed_form}.

\paragraph{Minimizing the local quadratic approximation.}
The local quadratic model above turns the sensitivity analysis into a practical update rule for the teacher-side inverse temperature. Standard TTM optimizes the student parameters for a fixed teacher-side temperature. In contrast, our method treats the teacher-side inverse temperature as a sample-wise auxiliary variable and refines it for the current student. Conceptually, this gives a double-minimization view~\cite{markovkd,unditillable,cmi} of the TTM objective:
\begin{equation}
\min_{\theta_s}\min_{\gamma>0}
\left[
\mathcal{L}_{\mathrm{CE}}(y,q_{\theta_s})
+\beta D_\mathrm{KL}\!\bigl(p^{(\gamma)}\|q_{\theta_s}\bigr)
\right].
\label{eq:double_min_view}
\end{equation}
This expression is used only to motivate the update direction; we do not solve the inner problem to convergence. In practice, after warm-up, we keep the current teacher and student logits fixed and update $\gamma$ periodically every $I_{\gamma}$ epochs using one local step based on the quadratic approximation in Eq.~\eqref{eq:taylor_KL}. Dropping the constant term that does not affect the update, we minimize the quadratic surrogate
\begin{equation}
\Delta\gamma^\star
\triangleq
\arg\min_{\Delta\gamma\in\mathbb{R}}\;
\left[
\frac{d}{d\gamma}D_\mathrm{KL}\!\bigl(p^{(\gamma)}\|q\bigr)\,\Delta\gamma
+\frac{1}{2}\frac{d^2}{d\gamma^2}D_\mathrm{KL}\!\bigl(p^{(\gamma)}\|q\bigr)\,(\Delta\gamma)^2
\right].
\label{eq:psi_dgamma}
\end{equation}

\paragraph{Curvature-aware update.}
Because the surrogate in Eq.~\eqref{eq:psi_dgamma} is quadratic in $\Delta\gamma$, its minimizer has a closed form:
\begin{equation}
\Delta\gamma^\star
=
-\frac{\frac{d}{d\gamma}D_\mathrm{KL}\!\bigl(p^{(\gamma)}\|q\bigr)}
{\frac{d^2}{d\gamma^2}D_\mathrm{KL}\!\bigl(p^{(\gamma)}\|q\bigr)}.
\label{eq:delta_gamma_star}
\end{equation}
We update the inverse temperature using a step size $\eta>0$, add a small constant to the denominator for numerical stability, and clip the result to a valid range:
\begin{equation}
\gamma \leftarrow
\mathrm{clip}\!\left(
\gamma
-\eta\,
\frac{
\frac{d}{d\gamma}D_\mathrm{KL}\!\bigl(p^{(\gamma)}\|q\bigr)
}{
\frac{d^2}{d\gamma^2}D_\mathrm{KL}\!\bigl(p^{(\gamma)}\|q\bigr)+\delta
},
\gamma_{\min},\gamma_{\max}
\right),
\label{eq:gamma_update}
\end{equation}
where $\delta>0$ is a small constant for numerical stability, $\mathrm{clip}(x,a,b)=\min(\max(x,a),b)$, and $\gamma_{\min},\gamma_{\max}$ are the lower and upper bounds that restrict $\gamma$ to a valid range. In Sec.~\ref{ssec:closed_form}, we show that these derivatives can be written in compact variance-covariance forms, enabling an efficient per-sample update.

\paragraph{Temperature-adaptive extensions of TTM and WTTM.}
The update in Eq.~\eqref{eq:gamma_update} can be directly incorporated into TTM (Eq.~\eqref{eq:L_ttm}) and WTTM (Eq.~\eqref{eq:L_wttm}) by replacing the fixed inverse temperature with a sample-wise value updated during training. We denote the resulting methods as \emph{Temperature-Adaptive TTM (TA-TTM)} and \emph{Temperature-Adaptive WTTM (TA-WTTM)}, respectively. In both cases, we keep the original loss definitions and introduce only the additional $\gamma$ update computed with \texttt{stop-gradient} logits. We maintain $\gamma$ as a per-sample state throughout training and update it for observed samples during scheduled update epochs. For TA-WTTM, one could alternatively update $\gamma$ by minimizing the full weighted WTTM term, including the sample weight $U_{\gamma}(p)$. However, this couples temperature adaptation with sample weighting, and our ablation in Sec.~\ref{ssec:ablation_tawttm_full} shows that this full-objective variant underperforms the decoupled update in Eq.~\eqref{eq:gamma_update}. We therefore use the decoupled update as TA-WTTM in the rest of the paper. As a result, TA-TTM adapts the teacher-side temperature within the TTM objective, and TA-WTTM augments WTTM with the same temperature adaptation while retaining its original sample weighting scheme.

\subsection{Closed-Form Derivatives}
\label{ssec:closed_form}
The curvature-aware update in Sec.~\ref{ssec:overview} requires the first and second derivatives of the per-sample discrepancy $D_\mathrm{KL}\!\bigl(p^{(\gamma)}\|q\bigr)$ with respect to the teacher-side inverse temperature $\gamma$. Notably, both derivatives can be written using low-order moments of teacher and student logits under the transformed teacher weighting.

We recall $p^{(\gamma)}=\mathrm{softmax}(\gamma\mathbf{z}_t)$ and $q=\mathrm{softmax}(\mathbf{z}_s)$, and we differentiate with respect to $\gamma$ with stop-gradient logits. Let $m_t \triangleq \mathbb{E}_{p^{(\gamma)}}[z_t]$ and $m_s \triangleq \mathbb{E}_{p^{(\gamma)}}[z_s]$, where $\mathbb{E}_{p^{(\gamma)}}[f]\triangleq \sum_{i=1}^K p^{(\gamma)}_i f_i$ denotes the class-wise expectation. We write centered logits as $\hat{z}_t \triangleq z_t-m_t$ and $\hat{z}_s \triangleq z_s-m_s$ (element-wise). With these definitions, the first derivative can be written in a compact variance-covariance form under $p^{(\gamma)}$, while the second derivative yields a curvature expression that also remains a simple class-wise weighted sum. Concretely,
\begin{equation}
\begin{aligned}
\frac{d}{d\gamma}D_\mathrm{KL}\!\bigl(p^{(\gamma)}\|q\bigr)
&=
\underbrace{\gamma\,\mathrm{Var}_{p^{(\gamma)}}(z_t)}_{\text{\shortstack{teacher dispersion}}}
-
\underbrace{\mathrm{Cov}_{p^{(\gamma)}}(z_s,z_t)}_{\text{\shortstack{teacher-student\\alignment}}},\\
\frac{d^2}{d\gamma^2}D_\mathrm{KL}\!\bigl(p^{(\gamma)}\|q\bigr)
&=
\underbrace{\mathrm{Var}_{p^{(\gamma)}}(z_t)}_{\text{\shortstack{base curvature}}}
-
\underbrace{\mathrm{Cov}_{p^{(\gamma)}}(z_s,\hat{z}_t^{\,2})}_{\text{\shortstack{student-teacher\\deviation correlation}}}
+
\underbrace{\gamma\,\mathbb{E}_{p^{(\gamma)}}[\hat{z}_t^{\,3}]}_{\text{\shortstack{teacher skewness}}},
\end{aligned}
\label{eq:kl_derivatives}
\end{equation}
where $\mathrm{Var}_{p^{(\gamma)}}(\cdot)$ and $\mathrm{Cov}_{p^{(\gamma)}}(\cdot,\cdot)$ are taken with respect to the weighting $p^{(\gamma)}$. We provide the full derivations in the supplementary material.

\paragraph{Interpretation of first derivative.}
Eq.~\eqref{eq:kl_derivatives} shows that temperature adaptation is governed by simple statistics of teacher and student logits under the transformed teacher weighting $p^{(\gamma)}$. In the first derivative, the update direction is determined by the balance between teacher dispersion and teacher-student alignment. The term $\gamma\,\mathrm{Var}_{p^{(\gamma)}}(z_t)$ grows when the transformed teacher distribution emphasizes classes with widely spread teacher logits, meaning that changing $\gamma$ strongly affects the smoothness of $p^{(\gamma)}$. In contrast, $\mathrm{Cov}_{p^{(\gamma)}}(z_s,z_t)$ increases when the student logits align with the teacher logits on those emphasized classes, which counteracts the push to change $\gamma$.

\paragraph{Interpretation of second derivative.}
The second derivative quantifies the local curvature that determines how aggressive the update should be. The base curvature $\mathrm{Var}_{p^{(\gamma)}}(z_t)$ provides the dominant scale, while the student-teacher deviation correlation term $\mathrm{Cov}_{p^{(\gamma)}}(z_s,\hat{z}_t^{\,2})$ adjusts this curvature depending on how the student weights classes where the teacher deviates strongly from its weighted mean. Finally, the teacher skewness term $\gamma\,\mathbb{E}_{p^{(\gamma)}}[\hat{z}_t^{\,3}]$ captures asymmetry of teacher logits under $p^{(\gamma)}$.
\input{algo/algo_update}

\paragraph{Connection to implicit R\'enyi entropy regularization.}
This sample-wise adaptation also has an interpretation through the implicit R\'enyi entropy regularization in TTM (Eq.~\eqref{eq:L_ttm}). Since the teacher-side temperature determines the transformed teacher distribution, it also controls how strongly the student is implicitly encouraged to adjust its predictive entropy under the TTM objective. Therefore, updating $\gamma$ per sample not only changes the teacher target but also modulates the strength of the implicit R\'enyi entropy regularization across samples, allowing the regularization effect to adapt to sample difficulty and teacher-student mismatch.

\paragraph{Curriculum learning perspective.}
Our per-sample temperature adaptation can also be interpreted as an implicit curriculum. The inverse temperature $\gamma$ controls the sharpness of the transformed teacher distribution, which affects how difficult the distillation target is for the student to match. By updating $\gamma$ to reduce the teacher-student discrepancy, the method can assign softer targets to samples with large mismatch and sharper targets to samples where the student is more aligned with the teacher. In this sense, the teacher signal is adjusted at the sample level according to the current learning state, resembling curriculum or self-paced learning~\cite{cl,spl,sspl,flatcl}. Through Eq.~\eqref{eq:L_ttm}, this view is consistent with sample-wise modulation of the implicit R\'enyi entropy regularization, since changing $\gamma$ simultaneously changes the strength of the regularization effect induced by TTM.

\subsection{Algorithm and Implementation}
\label{ssec:algo}
Algorithm~\ref{alg:gamma} summarizes the per-sample inverse-temperature update. Each scheduled update requires one softmax over $K$ classes and a constant number of $O(K)$ weighted sums, yielding negligible overhead. We initialize $\gamma$ from the temperature used in TTM, keep it fixed during a short warm-up, and then update it every $I_{\gamma}$ epochs, reusing the stored values between updates. For a training set of $N$ samples, the per-sample state consists of $N$ 32-bit floating-point (FP32) scalars, requiring $4N$ bytes (0.2 MB for CIFAR100 and 5.1 MB for ImageNet1k). The state is indexed by dataset sample IDs, making it compatible with shuffling and standard data augmentation. The interval $I_{\gamma}$ trades update frequency against noise and computation. For stability, we add $\delta$ to the curvature denominator and clip $\gamma$ to $[\gamma_{\min},\gamma_{\max}]$; these safeguards handle small or negative estimated curvature and keep the stored inverse temperatures within the valid range.

\input{tabs/tabs_cifar100_same}
\input{tabs/tabs_cifar100_diff}

\section{Evaluation}
\label{sec:evaluation}

\subsection{Evaluation Setting}
\paragraph{Benchmarks and comparison methods.}
We evaluate our method on CIFAR100~\cite{cifar} and ImageNet1k~\cite{imagenet}, which are standard image classification benchmarks for distillation. CIFAR100 contains 100 classes of $32\times 32$ images, and ImageNet1k contains 1,000 classes of high-resolution natural images. We use widely adopted teacher-student architectures and report top-1 accuracy on the CIFAR100 test split and the ImageNet1k validation split. We compare against TTM~\cite{ttm}, WTTM~\cite{ttm}, strong KD baselines FitNet~\cite{fitnet}, AT~\cite{at}, VID~\cite{vid}, RKD~\cite{rkd}, PKT~\cite{pkt}, CRD~\cite{crd}, DIST~\cite{dist}, and DKD~\cite{dkd}, temperature-adaptive KD methods LS~\cite{ls} and EA-KD~\cite{eakd}, and the curriculum-based KD method CTKD~\cite{ctkd}.

\paragraph{Training and evaluation protocol.}
We follow the training and evaluation protocols of prior distillation work~\cite{crd} to ensure consistent comparisons. Since our method extends TTM and WTTM with temperature adaptation, we reproduce TTM and WTTM using the official implementation\footnote{https://github.com/zkxufo/TTM} under the protocol described in the original paper. For our method, all hyperparameters follow the corresponding TTM or WTTM configuration, except for the step size introduced for temperature adaptation. For the reproduced baselines and proposed methods in the main CIFAR100 experiments, we report the mean and standard deviation over five runs with different random seeds. Full details are provided in the supplementary material.

\subsection{Comparison Results}
\paragraph{CIFAR100 results.}
Tables~\ref{tab:cifar100_same} and~\ref{tab:std_diff} summarize results on CIFAR100 under same-architecture and cross-architecture distillation, respectively. TA-TTM is generally competitive with or better than the corresponding TTM baseline, while TA-WTTM consistently strengthens WTTM. The trend holds in both architecture settings, indicating that per-sample temperature adaptation is complementary to teacher-side temperature matching and remains useful when the teacher-student gap becomes larger.
TTM and WTTM already use carefully tuned $\beta$ and $\gamma$ configurations, so the improvements obtained by changing only the teacher-side temperature adaptation should be interpreted as modest refinements over strong baselines.

\paragraph{ImageNet1k results.}
Table~\ref{tab:ImageNet_main} reports results on ImageNet1k for the ResNet-34 teacher and ResNet-18 student pair. TA-TTM improves over the reproduced TTM baseline, and TA-WTTM further improves over WTTM while remaining competitive with the temperature-adaptive KD baseline EA-KD. These results indicate that the proposed sample-wise temperature adaptation remains effective at ImageNet1k scale and that its benefit is preserved when combined with the WTTM sample-weighting scheme.

\input{tabs/tabs_imagenet}
\input{tabs/tabs_cifar100_same_noce}
\paragraph{Distillation without cross-entropy.}
Table~\ref{tab:cifar100_same_noce} evaluates whether the proposed temperature adaptation remains effective when the supervised cross-entropy term is removed from the TTM-style objectives on representative pairs, which reflects recent practical settings where ground-truth labels are not available and learning must rely on teacher-provided supervision. Both TA-TTM and TA-WTTM retain their advantage over the corresponding non-adaptive baselines. These results indicate that the gains are not solely due to interaction with the supervised cross-entropy term; the adaptive teacher-side transform also strengthens the distillation signal itself.

\input{tabs/tabs_cifar100_vit}
\paragraph{Distillation from transformer-based teachers.}
Table~\ref{tab:cifar100_transformer} evaluates distillation performance from transformer-based teachers (ViT-S~\cite{vit}, Swin-T~\cite{swin}, and Mixer-B/16~\cite{mixer}) to a CNN-based student (ResNet18~\cite{resnet}). TTM and WTTM are strong baselines in this setting, and our adaptive variants further improve them. This suggests that per-sample temperature adaptation is also useful when transferring knowledge across model families, where the teacher-student gap is pronounced.

\subsection{Analysis of Temperature Dynamics}
\input{figs/fig_temperature_diagnostics_tattm}
\input{figs/fig_temperature_diagnostics_tawttm}
\paragraph{Adaptation behavior of the average temperature.}
We analyze the training dynamics of the temperature. Figs.~\ref{fig:temperature-diagnostics-tattm}(a) and~\ref{fig:temperature-diagnostics-tawttm}(a) show that the mean inverse temperature $\gamma$ evolves in a structured manner rather than converging to a constant. For both TA-TTM and TA-WTTM, $\gamma$ decreases from its initialization and then stabilizes after mid training, indicating that the effective temperature $T=1/\gamma$ increases and the teacher targets become softer over time. The red dashed curves plot the mixing coefficient $\lambda$ implied by the learned $\gamma$ under a fixed $\beta$, and their evolution suggests that the balance between the distillation term and the supervised cross-entropy term changes adaptively as a consequence of temperature adaptation. In addition, the standard deviation of $\gamma$ tends to increase, suggesting that sample-wise temperatures become more diverse as training progresses. Corresponding cross-architecture diagnostics are provided in the supplementary material.

\paragraph{Stage-wise temperature distribution.}
Figs.~\ref{fig:temperature-diagnostics-tattm}(b) and~\ref{fig:temperature-diagnostics-tawttm}(b) visualize how the inverse temperature $\gamma$ is distributed over training samples at three stages of training. At epoch 100, the distributions are sharply peaked for both methods, indicating that most samples share a similar temperature early in training. By epoch 200, the histograms become noticeably broader and their modes shift to smaller $\gamma$ values, which corresponds to using a larger effective temperature $T=1/\gamma$ and thus softer teacher targets on average. This widening suggests that the proposed update progressively differentiates samples and assigns a wider range of temperatures rather than converging to a single global value. From epoch 200 to epoch 240, the distributions change more mildly, implying that the sample-wise temperatures stabilize in the late stage.

\paragraph{Correlation with sample difficulty.}
Figs.~\ref{fig:temperature-diagnostics-tattm}(c) and~\ref{fig:temperature-diagnostics-tawttm}(c) examine how the optimized inverse temperature $\gamma$ relates to sample difficulty, measured by the per-sample cross-entropy loss, which is widely used as a difficulty proxy in curriculum and self-paced learning~\cite{spl}. Across both methods and training stages, we observe a clear negative correlation: samples with larger cross-entropy loss tend to be assigned smaller $\gamma$. Since smaller $\gamma$ corresponds to a larger effective temperature $T=1/\gamma$, the proposed adaptation softens the teacher target for harder samples while keeping it sharper for easier ones. This relationship remains pronounced throughout training, as reflected by the consistent downward trends in the point clouds. Overall, these results support a curriculum-like behavior in which the supervision is adjusted according to the current sample difficulty and teacher-student mismatch.

\subsection{Ablation Results}
\input{tabs/tabs_temperature_design_ablation}
\input{tabs/tabs_cifar100_wttm_ablation_short}

\paragraph{Temperature granularity.}
Table~\ref{tab:temperature_design_ablation}(a) compares global, class-wise, and sample-wise inverse temperatures on CIFAR100 for WRN-40-2$\rightarrow$WRN-16-2 using a common $\eta=0.001$. For both methods, the sample-wise variant achieves the highest top-1 best accuracy, outperforming the global and class-wise alternatives. This representative ablation supports adapting the teacher transform at the sample level, where teacher-student mismatch can vary even within the same class.

\paragraph{Temperature optimization rule.}
Table~\ref{tab:temperature_design_ablation}(b) fixes sample-wise $\gamma$ and compares trainable, first-order, and Newton updates over five runs using method-specific settings. The Newton update performs best for both TA-TTM and TA-WTTM, improving over directly training $\gamma$ and updating it with only the first derivative. On this representative pair, the result supports using the closed-form curvature rather than treating the inverse temperature as a generic trainable state.

\paragraph{Full-objective temperature update.}
\label{ssec:ablation_tawttm_full}
Table~\ref{tab:cifar100_wttm_ablation} compares two ways of updating the sample-wise inverse temperature in TA-WTTM on representative pairs. The decoupled variant, used as our default, updates $\gamma$ by minimizing the teacher-student discrepancy $D_\mathrm{KL}(p^{(\gamma)}\|q)$, while retaining the WTTM sample weight only in the training loss. The full-objective variant instead updates $\gamma$ using the weighted WTTM objective, thereby coupling the temperature update with the sample-weighting term $U_{\gamma}(p)$. TA-WTTM outperforms the full-objective variant in both same- and cross-architecture settings, indicating that separating temperature adaptation from sample weighting provides a more stable update for $\gamma$.

\input{figs/fig_hyperparameter_sensitivity}

\paragraph{Sensitivity to hyperparameters.}
Fig.~\ref{fig:hyperparameter-sensitivity} summarizes the sensitivity to the update rate $\eta$, warm-up length, and update interval. The results are relatively stable across warm-up lengths and update intervals, indicating that the method does not rely on a specific start epoch or very frequent temperature updates. The update rate has a clearer effect: overly conservative updates can under-adapt the temperature, while a moderate rate gives better accuracy. Overall, the proposed update is robust to scheduling choices, with $\eta$ being the main hyperparameter to tune.

\section{Conclusion}
We addressed how to adapt the teacher temperature in Transformed Teacher Matching (TTM) and proposed a per-sample update of the teacher-side inverse temperature within the TTM framework. Using closed-form first and second derivatives, expressed through variance-covariance statistics of centered teacher and student logits under the transformed teacher weighting, we derived an efficient curvature-aware update with negligible overhead. Experiments on standard image classification distillation benchmarks show that TA-TTM and TA-WTTM generally improve the reproduced TTM and WTTM baselines across multiple CIFAR100 settings, and also outperform the compared temperature-adaptive baseline in the ImageNet1k setting. Our analysis further reveals structured temperature dynamics and a negative correlation between $\gamma$ and per-sample difficulty, suggesting softer teacher targets for harder samples and a curriculum-like modulation of the implicit R\'enyi entropy regularization in TTM.

\bibliography{main}

\newpage
\appendix

\section{Derivations for Per-Sample Temperature Adaptation}
\label{app:deriv}

\subsection{Sample-wise optimization of the teacher temperature}
\label{app:ssec:setup}
We interpret the teacher temperature $T$ in TTM via $\gamma=1/T$ and optimize $\gamma$ per sample. Let the teacher-side transformed distribution and the student prediction be
\begin{equation}
p^{(\gamma)}=\mathrm{softmax}(\gamma\,\mathbf{z}_t),\qquad
q=\mathrm{softmax}(\mathbf{z}_s),\qquad \gamma>0.
\label{app:eq:def_pq}
\end{equation}
For a fixed pair $(\mathbf{z}_t,\mathbf{z}_s)$, we minimize the TTM discrepancy
\begin{equation}
D_\mathrm{KL}\!\bigl(p^{(\gamma)}\|q\bigr)
=
\sum_{i=1}^K p^{(\gamma)}_i\log\frac{p^{(\gamma)}_i}{q_i}.
\label{app:eq:def_KL}
\end{equation}

\paragraph{Remark (stop-gradient).}
In the following derivations, $\mathbf{z}_t$ and $\mathbf{z}_s$ (hence $q$) are treated as constants when differentiating with respect to $\gamma$.
In implementation, this corresponds to computing the $\gamma$-update with stop-gradient on logits.

\subsection{Log-sum-exp identities and an expectation derivative lemma}
\label{app:ssec:identities}

We start from the $\textit{log-sum-exp}$ form of the softmax:
\begin{equation}
p^{(\gamma)}_i
=
\frac{\exp(\gamma z_{t,i})}{\sum_{j=1}^K \exp(\gamma z_{t,j})}
=
\exp\left(\gamma z_{t,i}-\log\sum_{j=1}^K \exp(\gamma z_{t,j})\right)
\label{app:eq:p_logsumexp}
\end{equation}
and
\begin{equation}
\log p^{(\gamma)}_i
=
\gamma z_{t,i}-\log\sum_{j=1}^K \exp(\gamma z_{t,j}).
\label{app:eq:logp_logsumexp}
\end{equation}
Differentiating $\log p^{(\gamma)}_i$ with respect to $\gamma$ yields
\begin{align}
\frac{\partial}{\partial\gamma}\log p^{(\gamma)}_i
&=
z_{t,i}
-\frac{\sum_{j=1}^K z_{t,j}\exp(\gamma z_{t,j})}{\sum_{j=1}^K \exp(\gamma z_{t,j})}
=
z_{t,i}-\sum_{j=1}^K p^{(\gamma)}_j z_{t,j}.
\label{app:eq:dlogp_raw}
\end{align}
Define the $p^{(\gamma)}$-weighted mean of teacher logits
\begin{equation}
m_t \triangleq \sum_{j=1}^K p^{(\gamma)}_j z_{t,j}=\mathbb{E}_{p^{(\gamma)}}[z_t].
\label{app:eq:mt}
\end{equation}
Then Eq.~\eqref{app:eq:dlogp_raw} becomes
\begin{equation}
\frac{\partial}{\partial\gamma}\log p^{(\gamma)}_i=z_{t,i} - m_t.
\label{app:eq:dlogp}
\end{equation}

If the teacher logit are normalized as centered teacher logit $\hat z_{t,i}\triangleq z_{t,i}-m_t$, then Eq.~\eqref{app:eq:dlogp_raw} is
\begin{equation}
\frac{\partial}{\partial\gamma}\log p^{(\gamma)}_i=\hat z_{t,i}.
\label{app:eq:dlogp_centered}
\end{equation}

For the derivative $\frac{\partial}{\partial\gamma}p^{(\gamma)}_i$, we obtain it by the chain rule and \textit{log-sum-exp} form:
\begin{equation}
\begin{aligned}
\frac{\partial}{\partial\gamma}p^{(\gamma)}_i
&= \frac{\partial}{\partial\gamma} \exp\left(\gamma z_{t,i}-\log\sum_{j=1}^K \exp(\gamma z_{t,j})\right) \\
&=p^{(\gamma)}_i \cdot \frac{\partial}{\partial\gamma} \left(\gamma z_{t,i}-\log\sum_{j=1}^K \exp(\gamma z_{t,j})\right) \\
&=p^{(\gamma)}_i \cdot \frac{\partial}{\partial\gamma} \log p^{(\gamma)}_i \\
&=p^{(\gamma)}_i\,\hat z_{t,i}.
\label{app:eq:dp}
\end{aligned}
\end{equation}

\begin{lemma}
For any vector $f=(f_i)_{i=1}^K$ independent of $\gamma$, define the class-wise expectation $\mathbb{E}_{p^{(\gamma)}}[f]\triangleq \sum_{i=1}^K p^{(\gamma)}_i f_i$. Then
\begin{equation}
\frac{d}{d\gamma}\mathbb{E}_{p^{(\gamma)}}[f]
=
\mathbb{E}_{p^{(\gamma)}}[f\,\hat z_t]
=
\mathrm{Cov}_{p^{(\gamma)}}(f,z_t),
\label{app:eq:lemma_expect}
\end{equation}
where products are element-wise and $\mathrm{Cov}_{p^{(\gamma)}}(f,z_t)$ is the covariance under the weighting $p^{(\gamma)}$.
\end{lemma}

\begin{proof}
Using Eq.~\eqref{app:eq:dp},
\begin{equation}
\begin{aligned}
\frac{d}{d\gamma}\mathbb{E}_{p^{(\gamma)}}[f]
&=
\sum_i \frac{d p^{(\gamma)}_i}{d\gamma}\,f_i \\
&=
\sum_i p^{(\gamma)}_i \hat z_{t,i} f_i \\
&=
\mathbb{E}_{p^{(\gamma)}}[f\,\hat z_t].
\end{aligned}
\end{equation}
Then we have
\begin{equation}
\begin{aligned}
\frac{d}{d\gamma}\mathbb{E}_{p^{(\gamma)}}[f]
&=
\sum_i p^{(\gamma)}_i \hat z_{t,i} f_i \\
&=
\sum_i p^{(\gamma)}_i (z_{t,i} -m_t) f_i \\
&=
\sum_i p^{(\gamma)}_i z_{t,i} f_i - \sum_i p^{(\gamma)}_i m_t f_i \\
&=
\sum_i p^{(\gamma)}_i z_{t,i} f_i - m_t \sum_i p^{(\gamma)}_i f_i \\
&=
\mathbb{E}_{p^{(\gamma)}}[f\,z_t] - \mathbb{E}_{p^{(\gamma)}}[z_t]\mathbb{E}_{p^{(\gamma)}}[f] \\
&=
\mathrm{Cov}_{p^{(\gamma)}}(f,z_t).
\end{aligned}
\end{equation} 
\end{proof}

\begin{corollary}
Taking $f=z_t$ in Eq.~\eqref{app:eq:lemma_expect} gives
\begin{equation}
\frac{d}{d\gamma}\mathbb{E}_{p^{(\gamma)}}[z_t]
=
\mathbb{E}_{p^{(\gamma)}}[\hat z_t^{\,2}]
=
\mathrm{Var}_{p^{(\gamma)}}(z_t).
\label{app:eq:mt_prime}
\end{equation}
\end{corollary}

\subsection{Expanding the KL discrepancy}
\label{app:ssec:expand}

From Eq.~\eqref{app:eq:def_KL},
\begin{equation}
\begin{aligned}
D_\mathrm{KL}\!\bigl(p^{(\gamma)}\|q\bigr)
&=
\sum_{i=1}^K p^{(\gamma)}_i\bigl(\log p^{(\gamma)}_i-\log q_i\bigr) \\
&=
\sum_{i=1}^K p^{(\gamma)}_i\left(\gamma z_{t,i}-\log\sum_{j=1}^K \exp(\gamma z_{t,j})-\log q_i\right) \\
&=
\gamma \sum_{i=1}^K p^{(\gamma)}_i z_{t,i} - \sum_{i=1}^K p^{(\gamma)}_i \log\sum_{j=1}^K \exp(\gamma z_{t,j}) - \sum_{i=1}^K p^{(\gamma)}_i \log q_i \\
&=
\gamma \sum_{i=1}^K p^{(\gamma)}_i z_{t,i} - \log\sum_{j=1}^K \exp(\gamma z_{t,j}) \sum_{i=1}^K p^{(\gamma)}_i - \sum_{i=1}^K p^{(\gamma)}_i \log q_i \\
&=
\gamma \mathbb{E}_{p^{(\gamma)}}[z_{t}] - \log\sum_{j=1}^K \exp(\gamma z_{t,j}) - \mathbb{E}_{p^{(\gamma)}}[\log q],
\label{app:eq:KL_as_expect}
\end{aligned}
\end{equation}
where we use $\sum_{i=1}^K p^{(\gamma)}_i=1$.

\subsection{First derivative}
\label{app:ssec:first}

Differentiate Eq.~\eqref{app:eq:KL_as_expect} with respect to $\gamma$:
\begin{equation}
\begin{aligned}
\frac{d}{d\gamma}D_\mathrm{KL}\!\bigl(p^{(\gamma)}\|q\bigr)
&=
\frac{d}{d\gamma} \left\{ \gamma \mathbb{E}_{p^{(\gamma)}}[z_{t}] - \log\sum_{j=1}^K \exp(\gamma z_{t,j}) - \mathbb{E}_{p^{(\gamma)}}[\log q] \right\} \\
&=
\frac{d}{d\gamma} \gamma \mathbb{E}_{p^{(\gamma)}}[z_{t}] - \frac{d}{d\gamma} \log\sum_{j=1}^K \exp(\gamma z_{t,j}) - \frac{d}{d\gamma} \mathbb{E}_{p^{(\gamma)}}[\log q]  \\
&=
\mathbb{E}_{p^{(\gamma)}}[z_{t}] + \gamma \frac{d}{d\gamma} \mathbb{E}_{p^{(\gamma)}}[z_{t}] - \mathbb{E}_{p^{(\gamma)}}[z_{t}] -  \frac{d}{d\gamma} \mathbb{E}_{p^{(\gamma)}}[\log q] \\
&=
\gamma \frac{d}{d\gamma} \mathbb{E}_{p^{(\gamma)}}[z_{t}] -  \frac{d}{d\gamma} \mathbb{E}_{p^{(\gamma)}}[\log q]
\label{app:eq:dKL_step1}
\end{aligned}
\end{equation}
By the lemma and corollary, we obtain
\begin{equation}
\begin{aligned}
\frac{d}{d\gamma}D_\mathrm{KL}\!\bigl(p^{(\gamma)}\|q\bigr)
&=
\gamma\,\mathbb{E}_{p^{(\gamma)}}[\hat z_t^{\,2}]
-\mathbb{E}_{p^{(\gamma)}}[\hat z_t\,\log q] \\
&=
\gamma \mathrm{Var}_{p^{(\gamma)}}(z_t) - \mathrm{Cov}_{p^{(\gamma)}}(\log q,z_t).
\label{app:eq:dKL_mid}
\end{aligned}
\end{equation}

Next, note that student log-probabilities satisfy $\log q_i=z_{s,i}-\log\sum_j e^{z_{s,j}}$. Since the second term is class-independent constant $c$, it does not affect covariance with $z_t$, hence 
\begin{equation}
\begin{aligned}
\mathrm{Cov}_{p^{(\gamma)}}(\log q,z_t)
&=\mathbb{E}_{p^{(\gamma)}}[(\log q)\,z_t]-\mathbb{E}_{p^{(\gamma)}}[\log q]\mathbb{E}_{p^{(\gamma)}}[z_t] \\
&=\mathbb{E}_{p^{(\gamma)}}[(z_s-c)\,z_t]-\mathbb{E}_{p^{(\gamma)}}[z_s-c]\mathbb{E}_{p^{(\gamma)}}[z_t] \\
&=\Bigl(\mathbb{E}_{p^{(\gamma)}}[z_s z_t]-c\,\mathbb{E}_{p^{(\gamma)}}[z_t]\Bigr)
-\Bigl(\mathbb{E}_{p^{(\gamma)}}[z_s]-c\Bigr)\mathbb{E}_{p^{(\gamma)}}[z_t] \\
&=\mathbb{E}_{p^{(\gamma)}}[z_s z_t]-\mathbb{E}_{p^{(\gamma)}}[z_s]\mathbb{E}_{p^{(\gamma)}}[z_t]
=\mathrm{Cov}_{p^{(\gamma)}}(z_s,z_t). 
\end{aligned}
\end{equation}

Therefore, we obtain
\begin{equation}
\frac{d}{d\gamma}D_\mathrm{KL}\!\bigl(p^{(\gamma)}\|q\bigr)
=
\gamma\,\mathrm{Var}_{p^{(\gamma)}}(z_t)-\mathrm{Cov}_{p^{(\gamma)}}(z_s,z_t).
\label{app:eq:dKL_final}
\end{equation}

\subsection{Second derivative}
\label{app:ssec:second}

Before differentiating Eq.~\eqref{app:eq:dKL_final}, we rearrange it using the centered logits $\hat z_t, \hat z_s$. In addition to $m_t$ and $\hat z_t$ defined above, let $m_s\triangleq\mathbb{E}_{p^{(\gamma)}}[z_s]$ and $\hat z_s\triangleq z_s-m_s$ element-wise. By the definitions of variance and covariance,
\begin{equation}
\begin{aligned}
\mathrm{Var}_{p^{(\gamma)}}(z_t)
&=\mathbb{E}_{p^{(\gamma)}}[(z_t-m_t)^2]
=\mathbb{E}_{p^{(\gamma)}}[\hat z_t^{\,2}],\\
\mathrm{Cov}_{p^{(\gamma)}}(z_s,z_t)
&=\mathbb{E}_{p^{(\gamma)}}[(z_s-m_s)(z_t-m_t)]
=\mathbb{E}_{p^{(\gamma)}}[\hat z_s\,\hat z_t].
\end{aligned}
\end{equation}
Substituting these results into first derivative yeilds the following expression:
\begin{equation}
\frac{d}{d\gamma}D_\mathrm{KL}\!\bigl(p^{(\gamma)}\|q\bigr)
=
\gamma\,\mathbb{E}_{p^{(\gamma)}}[\hat z_t^{\,2}]-\mathbb{E}_{p^{(\gamma)}}[\hat z_s\,\hat z_t].
\label{app:eq:dKL_centered}
\end{equation}

For calculating the second derivative, let 
\begin{equation}
a(\gamma)\triangleq \mathbb{E}_{p^{(\gamma)}}[\hat z_t^{\,2}],\qquad b(\gamma)\triangleq \mathbb{E}_{p^{(\gamma)}}[\hat z_s\,\hat z_t].
\end{equation}
Then $\frac{d}{d\gamma}D_\mathrm{KL}\!\bigl(p^{(\gamma)}\|q\bigr)=\gamma a(\gamma)-b(\gamma)$ and
\begin{equation}
\frac{d^2}{d\gamma^2}D_\mathrm{KL}\!\bigl(p^{(\gamma)}\|q\bigr)=a(\gamma)+\gamma a'(\gamma)-b'(\gamma).
\end{equation}
Although $\hat z_t,\hat z_s$ depend on $\gamma$ through $m_t,m_s$, the derivatives can be computed by expanding into these terms.

First, we expand $a(\gamma)$:
\begin{equation}
\begin{aligned}
a(\gamma)
=\mathbb{E}_{p^{(\gamma)}}[(z_t-m_t)^2]
&=\mathbb{E}_{p^{(\gamma)}}[z_t^2]-2m_t\mathbb{E}_{p^{(\gamma)}}[z_t]+m_t^2 \\
&=\mathbb{E}_{p^{(\gamma)}}[z_t^2]-m_t^2,
\end{aligned}
\end{equation}
since $\mathbb{E}_{p^{(\gamma)}}[z_t]=m_t$. Differentiating this equation gives
\begin{equation}
a'(\gamma)
=\frac{d}{d\gamma}\mathbb{E}_{p^{(\gamma)}}[z_t^2]-2m_t\,m_t'.
\end{equation}
Because $z_t^2$ is $\gamma$-independent, we can apply the lemma
\begin{equation}
\frac{d}{d\gamma}\mathbb{E}_{p^{(\gamma)}}[z_t^2]
=\mathbb{E}_{p^{(\gamma)}}[z_t^2\,\hat z_t].
\end{equation}
Similarly, applying the lemma with $f=z_t$ yields
\begin{equation}
\begin{aligned}
m_t'
&=\frac{d}{d\gamma}\mathbb{E}_{p^{(\gamma)}}[z_t]\\
&=\mathbb{E}_{p^{(\gamma)}}[z_t\,\hat z_t]\\
&=\mathbb{E}_{p^{(\gamma)}}[(z_t-m_t)z_t]\\
&=\mathbb{E}_{p^{(\gamma)}}\!\left[\hat z_t(\hat z_t+m_t)\right]\\
&=\mathbb{E}_{p^{(\gamma)}}[\hat z_t^{\,2}] + m_t\,\mathbb{E}_{p^{(\gamma)}}[\hat z_t]\\
&=\mathbb{E}_{p^{(\gamma)}}[\hat z_t^{\,2}] \qquad (\because~\mathbb{E}_{p^{(\gamma)}}[\hat z_t]=0)\\
&=a(\gamma).
\end{aligned}
\end{equation}
Substituting these results into $a'(\gamma)$ gives
\begin{equation}
\begin{aligned}
a'(\gamma)
&=\mathbb{E}_{p^{(\gamma)}}[z_t^2\,\hat z_t]-2m_t\,\mathbb{E}_{p^{(\gamma)}}[\hat z_t^{\,2}] \\
&=\mathbb{E}_{p^{(\gamma)}}\big[(\hat z_t+m_t)^2\,\hat z_t\big]-2m_t\,\mathbb{E}_{p^{(\gamma)}}[\hat z_t^{\,2}] \\
&=\mathbb{E}_{p^{(\gamma)}}\big[(\hat z_t^2+2m_t\hat z_t+m_t^2)\hat z_t\big]-2m_t\,\mathbb{E}_{p^{(\gamma)}}[\hat z_t^{\,2}] \\
&=\mathbb{E}_{p^{(\gamma)}}[\hat z_t^{\,3}]
+2m_t\,\mathbb{E}_{p^{(\gamma)}}[\hat z_t^{\,2}]
+m_t^2\,\mathbb{E}_{p^{(\gamma)}}[\hat z_t]
-2m_t\,\mathbb{E}_{p^{(\gamma)}}[\hat z_t^{\,2}] \\
&=\mathbb{E}_{p^{(\gamma)}}[\hat z_t^{\,3}],
\end{aligned}
\end{equation}
where we used $\mathbb{E}_{p^{(\gamma)}}[\hat z_t]=\mathbb{E}_{p^{(\gamma)}}[z_t-m_t]=0$.

Next, we derive the derivative of $b(\gamma)$. we expand 
\begin{equation}
\begin{aligned}
b(\gamma)
&=\mathbb{E}_{p^{(\gamma)}}[(z_s-m_s)(z_t-m_t)] \\
&=\mathbb{E}_{p^{(\gamma)}}[z_s z_t]-m_s\,\mathbb{E}_{p^{(\gamma)}}[z_t]-m_t\,\mathbb{E}_{p^{(\gamma)}}[z_s]+m_s m_t \\
&=\mathbb{E}_{p^{(\gamma)}}[z_s z_t]-m_s m_t,
\end{aligned}
\end{equation}
since $\mathbb{E}_{p^{(\gamma)}}[z_t]=m_t$ and $\mathbb{E}_{p^{(\gamma)}}[z_s]=m_s$. Differentiating this equation yields
\begin{equation}
b'(\gamma)
=\frac{d}{d\gamma}\mathbb{E}_{p^{(\gamma)}}[z_s z_t] - m_s' m_t - m_s m_t'.
\end{equation}
Because $z_s z_t$ is $\gamma$-independent (logits are treated as fixed w.r.t. $\gamma$), we apply the lemma with $f=z_s z_t$:
\begin{equation}
\frac{d}{d\gamma}\mathbb{E}_{p^{(\gamma)}}[z_s z_t]
=\mathbb{E}_{p^{(\gamma)}}[z_s z_t\,\hat z_t].
\end{equation}
Moreover, applying the lemma with $f=z_s$ gives
\begin{equation}
\begin{aligned}
m_s'
&=\frac{d}{d\gamma}\mathbb{E}_{p^{(\gamma)}}[z_s]  \\
&=\mathbb{E}_{p^{(\gamma)}}[z_s\,\hat z_t] \\
&=\mathbb{E}_{p^{(\gamma)}}\!\left[\bigl((z_s-m_s)+m_s\bigr)\hat z_t\right]\\
&=\mathbb{E}_{p^{(\gamma)}}[(z_s-m_s)\hat z_t] + m_s\,\mathbb{E}_{p^{(\gamma)}}[\hat z_t]\\
&=\mathbb{E}_{p^{(\gamma)}}[(z_s-m_s)\hat z_t] \qquad (\because~\mathbb{E}_{p^{(\gamma)}}[\hat z_t]=0)\\
&=\mathbb{E}_{p^{(\gamma)}}[\hat z_s\,\hat z_t]\\
&=b(\gamma).
\end{aligned}
\end{equation}
and we already have $m_t'=a(\gamma)$. Substituting these results into $b'(\gamma)$, we obtain
\begin{equation}
\begin{aligned}
b'(\gamma)
&=\mathbb{E}_{p^{(\gamma)}}[z_s z_t\,\hat z_t] - b(\gamma)m_t - m_s a(\gamma) \\
&=\mathbb{E}_{p^{(\gamma)}}\big[z_s(\hat z_t+m_t)\hat z_t\big] - m_t\,\mathbb{E}_{p^{(\gamma)}}[\hat z_s\hat z_t] - m_s\,\mathbb{E}_{p^{(\gamma)}}[\hat z_t^{\,2}] \\
&=\mathbb{E}_{p^{(\gamma)}}[z_s \hat z_t^{\,2}] + m_t\,\mathbb{E}_{p^{(\gamma)}}[z_s \hat z_t]
- m_t\,\mathbb{E}_{p^{(\gamma)}}[\hat z_s\hat z_t] - m_s\,\mathbb{E}_{p^{(\gamma)}}[\hat z_t^{\,2}] \\
&=\mathbb{E}_{p^{(\gamma)}}[z_s \hat z_t^{\,2}] - m_s\,a(\gamma),
\end{aligned}
\end{equation}
since $m_s'=b(\gamma)$. Finally, write $z_s=\hat z_s+m_s$:
\begin{equation}
\begin{aligned}
\mathbb{E}_{p^{(\gamma)}}[z_s \hat z_t^{\,2}] - m_s\,a(\gamma)
&=\mathbb{E}_{p^{(\gamma)}}[(\hat z_s+m_s)\hat z_t^{\,2}] - m_s\,\mathbb{E}_{p^{(\gamma)}}[\hat z_t^{\,2}] \\
&=\mathbb{E}_{p^{(\gamma)}}[\hat z_s\,\hat z_t^{\,2}].
\end{aligned}
\end{equation}

Finally, we obtain
\begin{equation}
\begin{aligned}
\frac{d^2}{d\gamma^2}D_\mathrm{KL}\!\bigl(p^{(\gamma)}\|q\bigr)
&=
\mathbb{E}_{p^{(\gamma)}}[\hat z_t^{\,2}]
+\gamma\,\mathbb{E}_{p^{(\gamma)}}[\hat z_t^{\,3}]
-\mathbb{E}_{p^{(\gamma)}}[\hat z_s\,\hat z_t^{\,2}] \\
&=\mathrm{Var}_{p^{(\gamma)}}(z_t)
-\mathrm{Cov}_{p^{(\gamma)}}(z_s,\hat z_t^{\,2})
+\gamma\,\mathbb{E}_{p^{(\gamma)}}[\hat z_t^{\,3}].
\label{app:eq:ddKL_final}
\end{aligned}
\end{equation}

\section{Comparison with Adaptive Temperature Distillation}
\label{app:sec:atd}
Adaptive Temperature Distillation (ATD)~\cite{yang2025atd} precomputes sample-wise temperatures from the difficulty estimated by a pretrained teacher. For a fair comparison of temperature-adaptation strategies, we use the CNN-RIS-AT setting from ATD on CIFAR100 with WRN-40-2$\rightarrow$WRN-16-2 and apply only its temperature regulation strategy, without Mixup.

\input{tabs/tabs_atd_comparison}

TA-TTM and TA-WTTM outperform ATD by 0.59 and 0.74 percentage points, respectively, while also improving over TTM and WTTM. The methods differ in when and how the sample-wise temperatures are determined: ATD computes them from pretrained-teacher difficulty before distillation, whereas our method maintains and repeatedly updates each sample's $\gamma$ during training from the current teacher--student mismatch within the TTM objective.

\section{Additional Diagnostics of Temperature Adaptation}
\label{app:sec:additional_diagnostics}

\subsection{Cross-architecture temperature diagnostics}
Figure~\ref{fig:temperature-diagnostics-cross-supp} complements the same-architecture diagnostics in the main paper with the cross-architecture pair WRN-40-2$\rightarrow$ShuffleNetV1. TA-TTM and TA-WTTM show the same qualitative pattern: the inverse-temperature distributions broaden during training, and larger per-sample cross-entropy losses tend to be associated with smaller $\gamma$.

\input{figs/fig_temperature_diagnostics_cross_supp}

\subsection{Curvature and clipping}
We examine the curvature used by the temperature update for WRN-40-2$\rightarrow$WRN-16-2. Table~\ref{tab:curvature_diagnostics} reports statistics of the estimated curvature $h_{\mathrm{update}}$ reconstructed over all 50,000 training samples at three update epochs. Negative curvature occurs primarily at epoch 100, but only for 0.164\% of samples for either method. It disappears for TA-TTM at epochs 200 and 240, and falls to 0.004\% and then 0\% for TA-WTTM. We observe no samples with $|h_{\mathrm{update}}|<10^{-8}$, the threshold associated with the denominator guard, and only 0.002\% with $|h_{\mathrm{update}}|<10^{-3}$ in one setting. Consistent with these curvature statistics, clipping is almost never activated. Thus, damping and clipping act mainly as numerical safeguards in this experiment rather than routinely determining the temperature update.

\input{tabs/tabs_curvature_diagnostics}

\subsection{Correlation between inverse temperature and sample difficulty}
We quantify the relationship visualized in the main-paper diagnostic figures using the same protocol of 1,000 samples per epoch. Table~\ref{tab:gamma_ce_correlation} reports both Pearson and Spearman correlations between the learned inverse temperature $\gamma$ and the per-sample cross-entropy loss for WRN-40-2$\rightarrow$WRN-16-2. All coefficients are negative at epochs 100, 200, and 240 for both methods, with Pearson correlations ranging from $-0.713$ to $-0.645$ and Spearman correlations from $-0.744$ to $-0.681$. Across all 13 settings used for the paper figures, the mean Pearson correlation is $-0.611$ for TA-TTM and $-0.595$ for TA-WTTM. These quantitative results corroborate the negative trend shown in the main-paper diagnostic figures and summarized in its conclusion: samples with higher cross-entropy loss tend to receive smaller $\gamma$, corresponding to softer teacher targets.

\input{tabs/tabs_gamma_ce_correlation}

\section{Hyperparameters}
For CIFAR100, Table~\ref{tab:hyperparams} summarizes the hyperparameters of our proposed method. We use the original TTM and WTTM configurations for the initial inverse temperature $\gamma_{\mathrm{init}}$ and $\beta$, and grid-search the additional step size $\eta$ over $\{0.001, 0.002, \ldots, 0.005\}$. For ImageNet1k, we use the TTM configuration and grid-search $\eta$ over the same range.

\input{tabs/tabs_hyperparams}
\end{document}

%% file: figs/fig_overview.tex
\begin{figure}[t]
\centering
\includegraphics[width=0.9\linewidth]{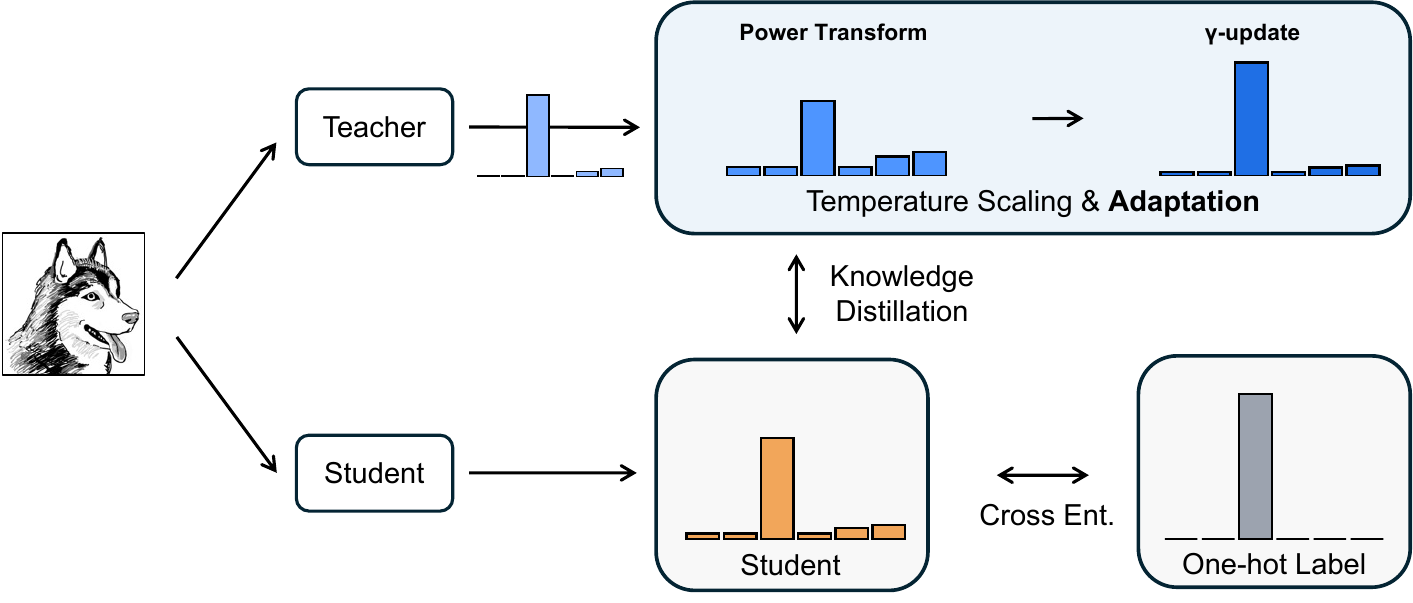}
\caption{Overview of the proposed Temperature-Adaptive Transformed Teacher Matching. Given an input sample, the teacher produces a predictive distribution, which is calibrated by a power transform with an adaptive inverse temperature $\gamma$. The inverse temperature $\gamma$ is updated to obtain sample-adaptive soft targets for distillation. The student is trained using cross-entropy and a distillation loss that matches the student prediction to the adapted teacher distribution.}
\label{fig:overview}
\end{figure}

%% file: algo/algo_update.tex
\begin{algorithm}[t]
\caption{Training with per-sample inverse-temperature adaptation}
\label{alg:gamma}
\begin{algorithmic}[1]
\REQUIRE Training data $\{(\mathbf{x},y)\}$, teacher $f_t$, student $f_s$, initial per-sample inverse temperatures $\{\gamma\}$, total epochs $E$, warm-up $E_{\mathrm{warm}}$, epoch update interval $I_{\gamma}$
\FOR{epoch $=1,\dots,E$}
  \FOR{each mini-batch $\mathcal{B}$}
    \STATE Compute teacher logits $\mathbf{z}_t=f_t(\mathbf{x})$ and student logits $\mathbf{z}_s=f_s(\mathbf{x})$
    \STATE Compute $p^{(\gamma)} \leftarrow \mathrm{softmax}(\gamma\,\mathbf{z}_t)$ and $q\leftarrow \mathrm{softmax}(\mathbf{z}_s)$
    \IF{epoch $> E_{\mathrm{warm}}$ and $(\mathrm{epoch}-E_{\mathrm{warm}}-1) \bmod I_{\gamma}=0$}
      \STATE $m_t \leftarrow \sum_i p^{(\gamma)}_i z_{t,i}$,\quad $m_s \leftarrow \sum_i p^{(\gamma)}_i z_{s,i}$
      \STATE $\hat{\mathbf{z}}_t \leftarrow \mathbf{z}_t - m_t$, \quad $\hat{\mathbf{z}}_s \leftarrow \mathbf{z}_s - m_s$
      \STATE Compute derivatives using Eq.~\eqref{eq:kl_derivatives} with detached logits
      \STATE $\gamma \leftarrow \texttt{clip}\!\left(\gamma-\eta\,\frac{\frac{d}{d\gamma}D_\mathrm{KL}\!\bigl(p^{(\gamma)}\|q\bigr)}{\frac{d^2}{d\gamma^2}D_\mathrm{KL}\!\bigl(p^{(\gamma)}\|q\bigr)+\delta},\gamma_{\min},\gamma_{\max}\right)$
    \ENDIF
    \STATE Compute the TTM/WTTM distillation loss and update student parameters
  \ENDFOR
\ENDFOR
\end{algorithmic}
\end{algorithm}

%% file: tabs/tabs_cifar100_same.tex
\begin{table}[t!]
\small
\setlength{\tabcolsep}{2pt}
\caption{Performance comparison on CIFAR100 in the same-architecture distillation setting. $\dagger$ indicates results reproduced by us using the official implementation under the protocol described in the paper~\cite{ttm}. $\Delta$ and $\Delta^{\dagger}$ denote accuracy differences from TTM/WTTM and TTM$^{\dagger}$/WTTM$^{\dagger}$, respectively; positive values are highlighted in red and negative values in blue.
}
\label{tab:cifar100_same}
\centering
\resizebox{\textwidth}{!}{
\begin{tabular}{cccccccc}
\toprule
\multicolumn{1}{c}{\multirow{2}{*}{Teacher}} 
& WRN-40-2 & WRN-40-2 & ResNet56 & ResNet110 & ResNet110 & ResNet32x4 & VGG13 \\
& 75.61 & 75.61 & 72.34 & 74.31 & 74.31 & 79.42 & 74.64 \\
\multicolumn{1}{c}{\multirow{2}{*}{Student}} 
& WRN-16-2 & WRN-40-1 & ResNet20 & ResNet20 & ResNet32 & ResNet8x4 & VGG8 \\
& 73.26 & 71.98 & 69.06 & 69.06 & 71.14 & 72.50 & 70.36 \\
\midrule
\multicolumn{8}{l}{\textit{Feature-based Knowledge Distillation}}\vspace{1mm} \\
FitNet & $73.58_{\pm 0.32}$ & $72.24_{\pm 0.24}$ & $69.21_{\pm 0.36}$ & $68.99_{\pm 0.27}$ & $71.06_{\pm 0.13}$ & $73.50_{\pm 0.28}$ & $71.02_{\pm 0.31}$ \\
AT & $74.08_{\pm 0.25}$ & $72.77_{\pm 0.10}$ & $70.55_{\pm 0.27}$ & $70.22_{\pm 0.16}$ & $72.31_{\pm 0.08}$ & $73.44_{\pm 0.19}$ & $71.43_{\pm 0.09}$ \\
VID & $74.11_{\pm 0.24}$ & $73.30_{\pm 0.13}$ & $70.38_{\pm 0.14}$ & $70.16_{\pm 0.39}$ & $72.61_{\pm 0.28}$ & $73.09_{\pm 0.21}$ & $71.23_{\pm 0.06}$ \\
RKD & $73.35_{\pm 0.09}$ & $72.22_{\pm 0.20}$ & $69.61_{\pm 0.06}$ & $69.25_{\pm 0.05}$ & $71.82_{\pm 0.34}$ & $71.90_{\pm 0.11}$ & $71.48_{\pm 0.05}$ \\
PKT & $74.54_{\pm 0.04}$ & $73.45_{\pm 0.19}$ & $70.34_{\pm 0.04}$ & $70.25_{\pm 0.04}$ & $72.61_{\pm 0.17}$ & $73.64_{\pm 0.18}$ & $72.88_{\pm 0.09}$ \\
CRD & $75.48_{\pm 0.09}$ & $74.14_{\pm 0.22}$ & $71.16_{\pm 0.17}$ & $71.46_{\pm 0.09}$ & $73.48_{\pm 0.13}$ & $75.51_{\pm 0.18}$ & $73.94_{\pm 0.22}$ \\
\midrule
\multicolumn{8}{l}{\textit{Logits-based Knowledge Distillation}}\vspace{1mm} \\
KD & $74.92_{\pm 0.28}$ & $73.54_{\pm 0.20}$ & $70.66_{\pm 0.24}$ & $70.67_{\pm 0.27}$ & $73.08_{\pm 0.18}$ & $73.33_{\pm 0.25}$ & $72.98_{\pm 0.19}$ \\
DIST & $75.51_{\pm 0.04 }$ & $74.73_{\pm 0.24 }$ & $71.75_{\pm 0.30 }$ & $71.65_{\pm 0.21 }$ & $73.69_{\pm 0.23 }$ & $76.31_{\pm 0.19 }$ & $73.89_{\pm 0.19}$ \\
DKD & 76.24 & 74.81 & 71.97 & n/a & 74.11 & 76.32 & 74.68 \\
TTM & $76.23_{\pm 0.15 }$ & $74.32_{\pm 0.31 }$ & $71.83_{\pm 0.16 }$ & $71.46_{\pm 0.16 }$ & $73.97_{\pm 0.23 }$ & $76.17_{\pm 0.28 }$ & $74.33_{\pm 0.07}$ \\
TTM$^{\dagger}$ & $76.15_{\pm 0.17}$ & $74.07_{\pm 0.33}$ & $71.89_{\pm 0.21}$ & $71.33_{\pm 0.16}$ & $73.83_{\pm 0.28}$ & $76.02_{\pm 0.18}$ & $74.18_{\pm 0.28}$ \\
\midrule
\multicolumn{8}{l}{\textit{Sample-Adaptive Knowledge Distillation}}\vspace{1mm} \\
WTTM & $76.37_{\pm 0.10 }$ & $74.58_{\pm 0.26 }$ & $71.92_{\pm 0.40 }$ & $71.67_{\pm 0.28 }$ & $74.13_{\pm 0.37 }$ & $76.06_{\pm 0.27 }$ & $74.44_{\pm 0.19}$ \\
WTTM$^{\dagger}$ & $76.26_{\pm 0.19}$ & $74.35_{\pm 0.24}$ & $71.77_{\pm 0.30}$ & $71.33_{\pm 0.30}$ & $73.93_{\pm 0.38}$ & $75.93_{\pm 0.31}$ & $74.33_{\pm 0.23}$ \\
\midrule
\multicolumn{8}{l}{\textit{Temperature-Adaptive Knowledge Distillation}}\vspace{1mm} \\
CTKD & $75.45$ & $73.93$ & $71.19$ & $70.99$ & $73.52$ & n/a & $73.52$ \\
LS & $76.11$ & $74.37$ & $71.43$ & $71.48$ & $74.17$ & $76.62$ & $74.36$ \\
EA-KD & n/a & $74.38$ & n/a & n/a & n/a & $75.46$ & $74.08$ \\
\midrule
\multicolumn{8}{l}{\textit{Temperature-Adaptive Transformed Teacher Matching}}\vspace{1mm} \\
\rowcolor{blue!8}
TA-TTM & $76.42_{\pm 0.16}$ & $74.23_{\pm 0.32}$ & $71.85_{\pm 0.05}$ & $71.46_{\pm 0.29}$ & $73.96_{\pm 0.42}$ & $76.33_{\pm 0.27}$ & $74.62_{\pm 0.13}$  \\
$\Delta$ & $\textcolor{red!70!black}{+0.19}$ & $\textcolor{blue!70!black}{-0.09}$ & $\textcolor{red!70!black}{+0.02}$ & $+0.00$ & $\textcolor{blue!70!black}{-0.01}$ & $\textcolor{red!70!black}{+0.16}$ & $\textcolor{red!70!black}{+0.29}$ \\
$\Delta^{\dagger}$ & $\textcolor{red!70!black}{+0.27}$ & $\textcolor{red!70!black}{+0.16}$ & $\textcolor{blue!70!black}{-0.04}$ & $\textcolor{red!70!black}{+0.13}$ & $\textcolor{red!70!black}{+0.13}$ & $\textcolor{red!70!black}{+0.31}$ & $\textcolor{red!70!black}{+0.44}$ \\
\rowcolor{blue!8}
TA-WTTM & $76.57_{\pm 0.17}$ & $74.49_{\pm 0.16}$ & $71.91_{\pm 0.22}$ & $71.73_{\pm 0.21}$ & $74.04_{\pm 0.11}$ & $76.28_{\pm 0.24}$ & $74.71_{\pm 0.08}$  \\
$\Delta$ & $\textcolor{red!70!black}{+0.20}$ & $\textcolor{blue!70!black}{-0.09}$ & $\textcolor{blue!70!black}{-0.01}$ & $\textcolor{red!70!black}{+0.06}$ & $\textcolor{blue!70!black}{-0.09}$ & $\textcolor{red!70!black}{+0.22}$ & $\textcolor{red!70!black}{+0.27}$ \\
$\Delta^{\dagger}$ & $\textcolor{red!70!black}{+0.31}$ & $\textcolor{red!70!black}{+0.14}$ & $\textcolor{red!70!black}{+0.14}$ & $\textcolor{red!70!black}{+0.40}$ & $\textcolor{red!70!black}{+0.11}$ & $\textcolor{red!70!black}{+0.35}$ & $\textcolor{red!70!black}{+0.38}$ \\
\bottomrule
\end{tabular}}
\end{table}

%% file: tabs/tabs_cifar100_diff.tex
\begin{table}[t!]
\small
\setlength{\tabcolsep}{2pt}
\caption{Performance comparison on CIFAR100 in the cross-architecture distillation setting.}
\label{tab:std_diff}
\centering
\resizebox{\textwidth}{!}{
\begin{tabular}{ccccccc}
\toprule
\multicolumn{1}{c}{\multirow{2}{*}{Teacher}} 
& VGG13 & ResNet50 & ResNet50 & ResNet32$\times$4 & ResNet32$\times$4 & WRN-40-2 \\
& 74.64 & 79.34 & 79.34 & 79.42 & 79.42 & 75.61 \\
\multicolumn{1}{c}{\multirow{2}{*}{Student}} 
& MobileNetV2 & MobileNetV2 & VGG8 & ShuffleNetV1 & ShuffleNetV2 & ShuffleNetV1 \\
& 64.60 & 64.60 & 70.36 & 70.50 & 71.82 & 70.50 \\
\midrule
\multicolumn{7}{l}{\textit{Feature-based Knowledge Distillation}}\vspace{1mm} \\
FitNet & $64.14_{\pm 0.50}$ & $63.16_{\pm 0.47}$ & $70.69_{\pm 0.22}$ & $73.59_{\pm 0.15}$ & $73.54_{\pm 0.22}$ & $73.73_{\pm 0.32}$ \\
AT & $59.40_{\pm 0.20}$ & $58.58_{\pm 0.54}$ & $71.84_{\pm 0.28}$ & $71.73_{\pm 0.31}$ & $72.73_{\pm 0.09}$ & $73.32_{\pm 0.35}$ \\
VID & $65.56_{\pm 0.42}$ & $67.57_{\pm 0.28}$ & $70.30_{\pm 0.31}$ & $73.38_{\pm 0.09}$ & $73.40_{\pm 0.17}$ & $73.61_{\pm 0.12}$ \\
RKD & $64.52_{\pm 0.45}$ & $64.43_{\pm 0.42}$ & $71.50_{\pm 0.07}$ & $72.28_{\pm 0.39}$ & $73.21_{\pm 0.28}$ & $72.21_{\pm 0.16}$ \\
PKT & $67.13_{\pm 0.30}$ & $66.52_{\pm 0.33}$ & $73.01_{\pm 0.14}$ & $74.10_{\pm 0.25}$ & $74.69_{\pm 0.34}$ & $73.89_{\pm 0.16}$ \\
CRD & $69.73_{\pm 0.42}$ & $69.11_{\pm 0.28}$ & $74.30_{\pm 0.14}$ & $75.11_{\pm 0.32}$ & $75.65_{\pm 0.10}$ & $76.05_{\pm 0.14}$ \\
\midrule
\multicolumn{7}{l}{\textit{Logits-based Knowledge Distillation}}\vspace{1mm} \\
KD & $67.37_{\pm 0.32}$ & $67.35_{\pm 0.32}$ & $73.81_{\pm 0.13}$ & $74.07_{\pm 0.19}$ & $74.45_{\pm 0.27}$ & $74.83_{\pm 0.17}$ \\
DIST & $68.50_{\pm 0.26}$ & $68.66_{\pm 0.23}$ & $74.11_{\pm 0.07}$ & $76.34_{\pm 0.18}$ & $77.35_{\pm 0.25}$ & $76.40_{\pm 0.03}$ \\
DKD & 69.71 & 70.35 & n/a & 76.45 & 77.07 & 76.70 \\
TTM & $68.98_{\pm 0.85}$ & $69.24_{\pm 0.28}$ & $74.87_{\pm 0.31}$ & $74.18_{\pm 0.26}$ & $76.57_{\pm 0.26}$ & $75.39_{\pm 0.33}$ \\
TTM$^{\dagger}$ & $68.65_{\pm 0.18}$ & $67.74_{\pm 0.71}$ & $74.73_{\pm 0.33}$ & $73.35_{\pm 0.17}$ & $76.32_{\pm 0.10}$ & $75.16_{\pm 0.09}$ \\
\midrule
\multicolumn{7}{l}{\textit{Sample-Adaptive Knowledge Distillation}}\vspace{1mm} \\
WTTM & $69.16_{\pm 0.20}$ & $69.59_{\pm 0.58}$ & $74.82_{\pm 0.28}$ & $74.37_{\pm 0.39}$ & $76.55_{\pm 0.08}$ & $75.42_{\pm 0.34}$ \\
WTTM$^{\dagger}$ & $68.07_{\pm 0.66}$ & $69.20_{\pm 0.16}$ & $74.56_{\pm 0.24}$ & $73.50_{\pm 0.29}$ & $76.70_{\pm 0.15}$ & $75.39_{\pm 0.17}$ \\
\midrule
\multicolumn{7}{l}{\textit{Temperature-Adaptive Knowledge Distillation}}\vspace{1mm} \\
CTKD & $68.46$ & $68.47$ & n/a & $74.48$ & $75.31$ & $75.78$ \\
LS & $68.61$ & $69.02$ & n/a & n/a & $75.56$ & n/a \\
EA-KD & $69.17$ & $69.67$ & n/a & n/a & $75.91$ & n/a \\

\midrule
\multicolumn{7}{l}{\textit{Temperature-Adaptive Transformed Teacher Matching}}\vspace{1mm} \\
\rowcolor{blue!8}
TA-TTM & $68.47_{\pm 0.50}$ & $68.86_{\pm 0.20}$ & $75.04_{\pm 0.08}$ & $74.28_{\pm 0.41}$ & $76.45_{\pm 0.14}$ & $75.61_{\pm 0.08}$ \\
$\Delta$ & $\textcolor{blue!70!black}{-0.51}$ & $\textcolor{blue!70!black}{-0.38}$ & $\textcolor{red!70!black}{+0.17}$ & $\textcolor{red!70!black}{+0.10}$ & $\textcolor{blue!70!black}{-0.12}$ & $\textcolor{red!70!black}{+0.22}$ \\
$\Delta^{\dagger}$ & $\textcolor{blue!70!black}{-0.18}$ & $\textcolor{red!70!black}{+1.12}$ & $\textcolor{red!70!black}{+0.31}$ & $\textcolor{red!70!black}{+0.93}$ & $\textcolor{red!70!black}{+0.13}$ & $\textcolor{red!70!black}{+0.45}$ \\
\rowcolor{blue!8}
TA-WTTM & $68.84_{\pm 0.74}$ & $69.51_{\pm 0.33}$ & $75.16_{\pm 0.29}$ & $74.32_{\pm 0.23}$ & $76.70_{\pm 0.23}$ & $75.72_{\pm 0.29}$ \\
$\Delta$ & $\textcolor{blue!70!black}{-0.32}$ & $\textcolor{blue!70!black}{-0.08}$ & $\textcolor{red!70!black}{+0.34}$ & $\textcolor{blue!70!black}{-0.05}$ & $\textcolor{red!70!black}{+0.15}$ & $\textcolor{red!70!black}{+0.30}$ \\
$\Delta^{\dagger}$ & $\textcolor{red!70!black}{+0.77}$ & $\textcolor{red!70!black}{+0.31}$ & $\textcolor{red!70!black}{+0.60}$ & $\textcolor{red!70!black}{+0.82}$ & $+0.00$ & $\textcolor{red!70!black}{+0.33}$ \\
\bottomrule
\end{tabular}}
\end{table}

%% file: tabs/tabs_imagenet.tex
\begin{table}[t!]
\small
\setlength{\tabcolsep}{4pt}
\caption{Performance comparison on ImageNet1k. TTM and WTTM results are from our reproduced experiments. These results are obtained from a single run.}
\label{tab:ImageNet_main}
\centering
\resizebox{0.8\textwidth}{!}{
\begin{tabular}{lcccccc}
\toprule
Teacher & \multicolumn{6}{c}{ResNet-34 (73.31)} \\
Student & \multicolumn{6}{c}{ResNet-18 (69.76)} \\
\midrule
Method & KD & TTM & WTTM & EA-KD & \cellcolor{blue!8}TA-TTM & \cellcolor{blue!8}TA-WTTM \\
Top-1 Acc. & 70.66 & 71.74 & 71.81 & 71.79 & \cellcolor{blue!8}71.92 ($\textcolor{red!70!black}{+0.18}$) & \cellcolor{blue!8}72.11 ($\textcolor{red!70!black}{+0.30}$) \\
\bottomrule
\end{tabular}}
\end{table}

%% file: tabs/tabs_cifar100_same_noce.tex
\begin{table}[t!]
\small
\setlength{\tabcolsep}{5pt}
\caption{Top-1 accuracy (\%) on CIFAR100 same-architecture pairs. TTM/WTTM variants omit cross-entropy; KD with cross-entropy is included as a reference. Results are reported over three seeds.}
\label{tab:cifar100_same_noce}
\centering
\begin{tabular}{c c c c}
\toprule
\multicolumn{1}{c}{Teacher}
& WRN-40-2 & ResNet56 & VGG13 \\
\multicolumn{1}{c}{Student}
& WRN-16-2 & ResNet20 & VGG8 \\
\midrule
KD w/ CE & $75.31_{\pm 0.25}$ & $71.52_{\pm 0.31}$ & $73.31_{\pm 0.28}$ \\
TTM w/o CE & $73.89_{\pm 0.19}$ & $71.21_{\pm 0.38}$ & $73.13_{\pm 0.30}$ \\
WTTM w/o CE & $75.27_{\pm 0.09}$ & $71.17_{\pm 0.02}$ & $74.02_{\pm 0.04}$ \\
\midrule
\rowcolor{blue!8}
TA-TTM w/o CE & $74.25_{\pm 0.20}$ & $71.29_{\pm 0.10}$ & $73.48_{\pm 0.23}$ \\
\rowcolor{blue!8}
TA-WTTM w/o CE & $75.59_{\pm 0.16}$ & $71.54_{\pm 0.41}$ & $74.14_{\pm 0.03}$ \\
\bottomrule
\end{tabular}
\end{table}

%% file: tabs/tabs_cifar100_vit.tex
\begin{table}[t!]
\small
\setlength{\tabcolsep}{7pt}
\caption{Performance comparison on CIFAR100 in the distillation setting from transformer-based teachers to CNN-based students.}
\label{tab:cifar100_transformer}
\centering
\begin{tabular}{c c c c}
\toprule
\multicolumn{1}{c}{Teacher}
& ViT-S (92.44) & Swin-T (89.26) & Mixer-B/16 (87.62) \\
\multicolumn{1}{c}{Student}
& ResNet18 (74.01) & ResNet18 (74.01) & ResNet18 (74.01) \\
\midrule
KD     & $76.41_{\pm 0.16}$ & $77.06_{\pm 0.17}$ & $76.84_{\pm 0.22}$ \\
EA-KD  & $76.82_{\pm 0.03}$ & $77.05_{\pm 0.18}$ & $76.81_{\pm 0.17}$ \\
TTM    & $78.89_{\pm 0.19}$ & $80.76_{\pm 0.15}$ & $79.04_{\pm 0.20}$ \\
WTTM   & $79.16_{\pm 0.11}$ & $80.37_{\pm 0.73}$ & $79.03_{\pm 0.17}$ \\
\midrule
\rowcolor{blue!8}
TA-TTM  & $79.31_{\pm 0.25}$ ($\textcolor{red!70!black}{+0.42}$) & $80.84_{\pm 0.07}$ ($\textcolor{red!70!black}{+0.08}$) & $79.29_{\pm 0.16}$ ($\textcolor{red!70!black}{+0.25}$) \\
\rowcolor{blue!8}
TA-WTTM & $79.55_{\pm 0.25}$ ($\textcolor{red!70!black}{+0.39}$) & $80.95_{\pm 0.24}$ ($\textcolor{red!70!black}{+0.58}$) & $79.59_{\pm 0.21}$ ($\textcolor{red!70!black}{+0.56}$) \\
\bottomrule
\end{tabular}
\end{table}

%% file: figs/fig_temperature_diagnostics_tattm.tex
\begin{figure*}[!t]
\centering

\begin{minipage}[t]{0.31\textwidth}
\centering
\includegraphics[width=\linewidth]{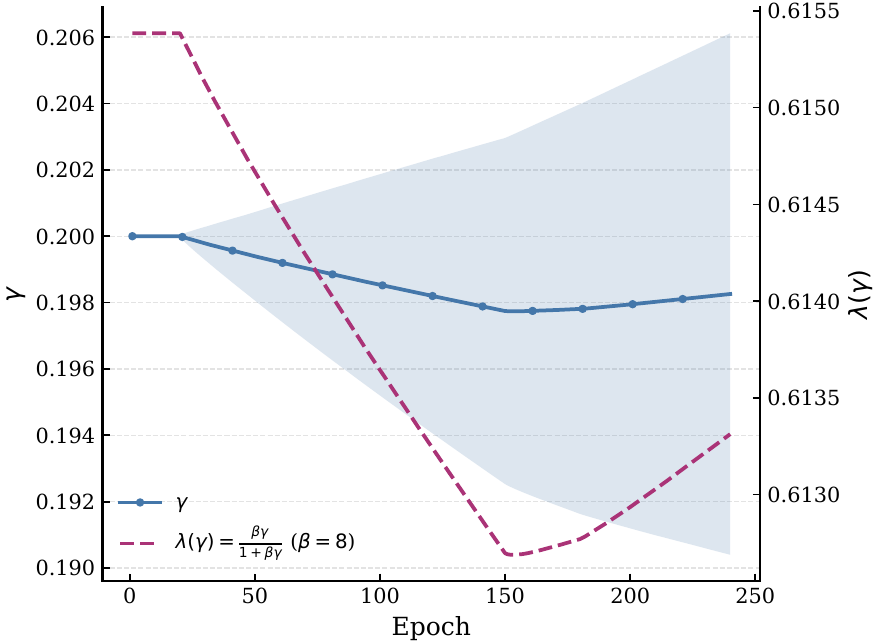}\\
(a) Dynamics
\end{minipage}\hfill
\begin{minipage}[t]{0.31\textwidth}
\centering
\includegraphics[width=\linewidth]{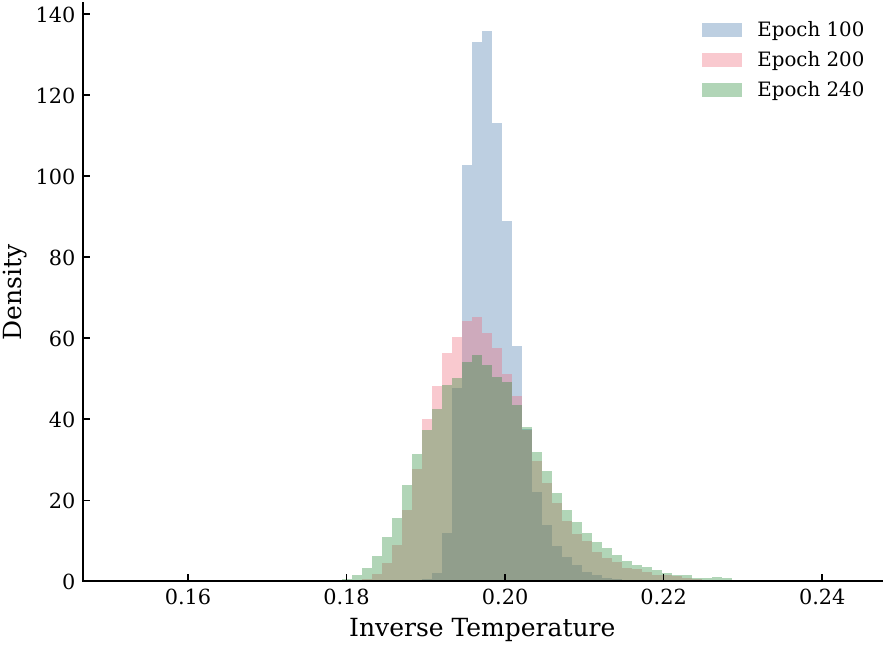}\\
(b) Distribution
\end{minipage}\hfill
\begin{minipage}[t]{0.31\textwidth}
\centering
\includegraphics[width=\linewidth]{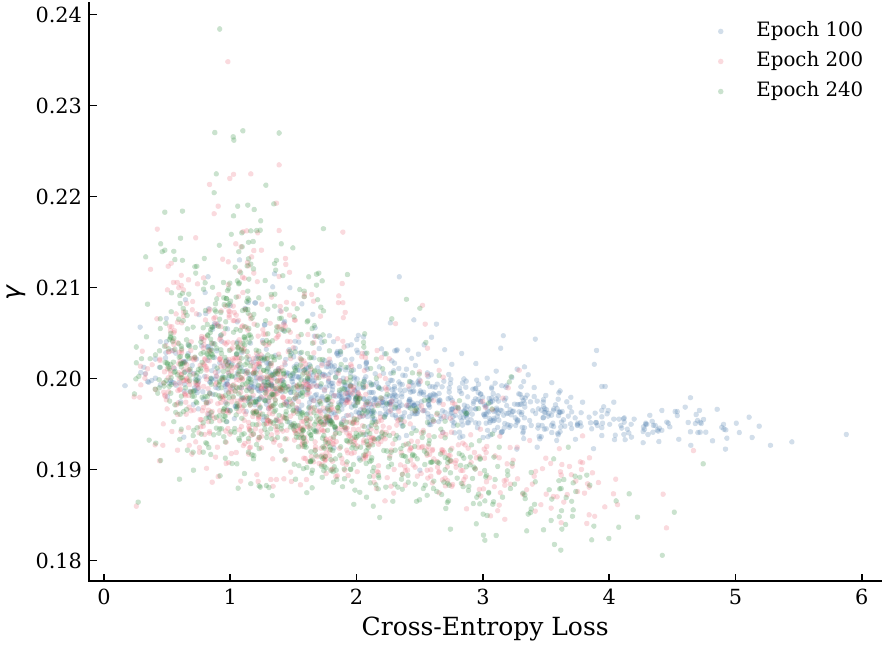}\\
(c) Difficulty correlation
\end{minipage}

\caption{Temperature diagnostics for TA-TTM on the representative same-architecture pair ResNet110$\rightarrow$ResNet20. The panels show the evolution of the mean and standard deviation of the inverse temperature $\gamma$, the stage-wise distribution of $\gamma$, and the correlation between per-sample cross-entropy loss and the optimized $\gamma$, respectively.}
\label{fig:temperature-diagnostics-tattm}
\end{figure*}

%% file: figs/fig_temperature_diagnostics_tawttm.tex
\begin{figure*}[!t]
\centering

\begin{minipage}[t]{0.31\textwidth}
\centering
\includegraphics[width=\linewidth]{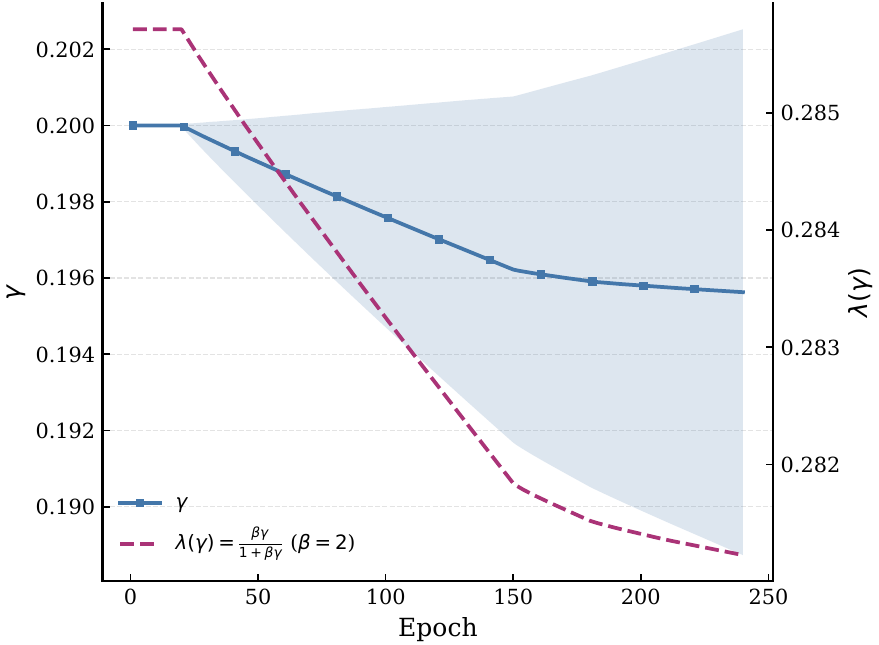}\\
(a) Dynamics
\end{minipage}\hfill
\begin{minipage}[t]{0.31\textwidth}
\centering
\includegraphics[width=\linewidth]{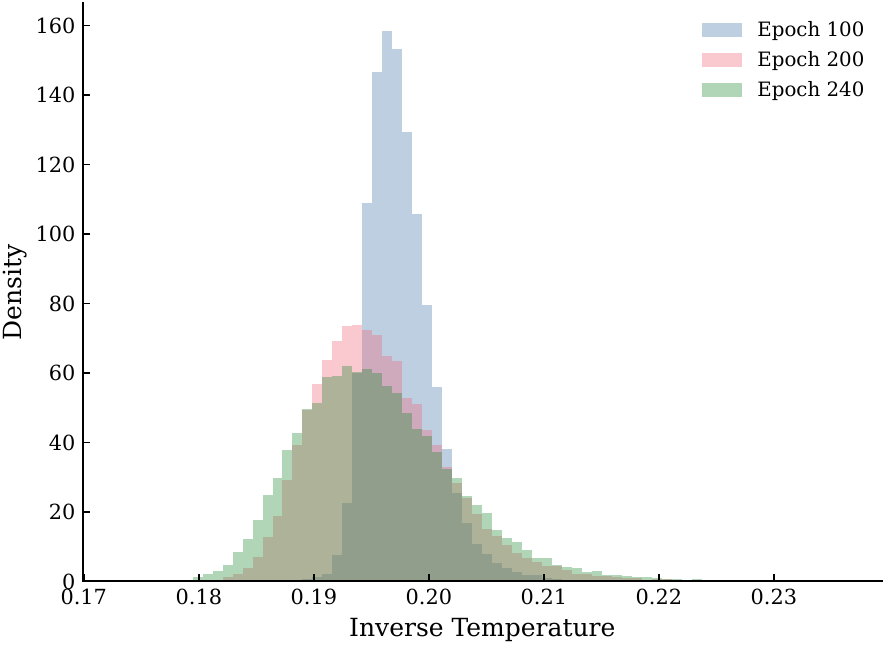}\\
(b) Distribution
\end{minipage}\hfill
\begin{minipage}[t]{0.31\textwidth}
\centering
\includegraphics[width=\linewidth]{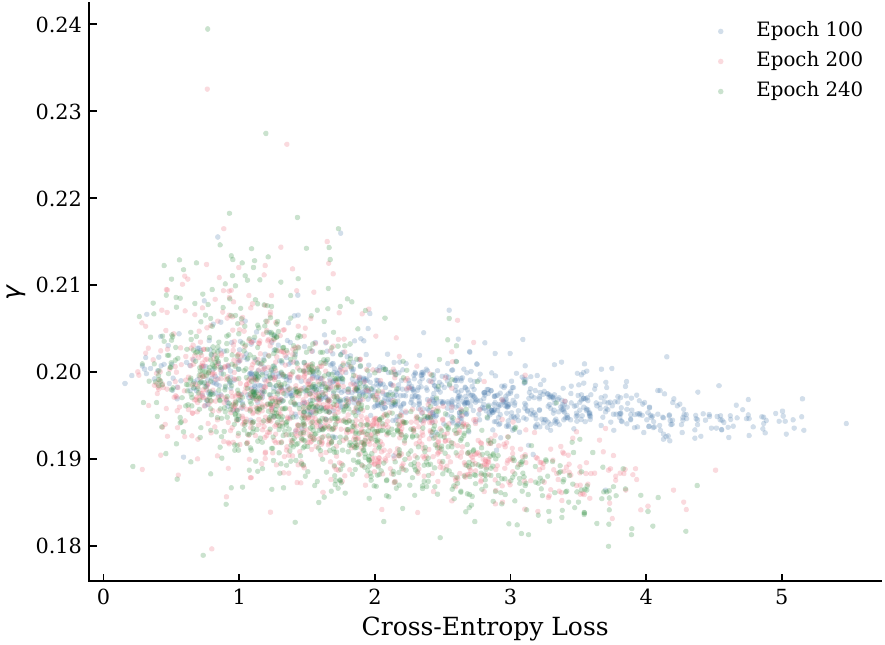}\\
(c) Difficulty correlation
\end{minipage}

\caption{Temperature diagnostics for TA-WTTM on ResNet110$\rightarrow$ResNet20, using the same layout as Fig.~\ref{fig:temperature-diagnostics-tattm}.}
\label{fig:temperature-diagnostics-tawttm}
\end{figure*}

%% file: tabs/tabs_temperature_design_ablation.tex
\begin{table*}[t!]
\centering
\small
\caption{Temperature-design ablations on CIFAR100 for WRN-40-2$\rightarrow$WRN-16-2. (a) compares inverse-temperature granularity using a common $\eta=0.001$ and reports top-1 best accuracy. (b) fixes sample-wise $\gamma$ and compares optimization rules using method-specific settings, reporting five-run top-1 best accuracy (mean $\pm$ std.).}
\begin{minipage}[t]{0.48\textwidth}
\centering
\textbf{(a) Temperature granularity}\\[3pt]
\setlength{\tabcolsep}{4pt}
\begin{tabular}{lccc}
\toprule
Method & Global $\gamma$ & Class $\gamma$ & Sample $\gamma$ \\
\midrule
TA-TTM  & 75.77 & 75.29 & \textbf{76.66} \\
TA-WTTM & 75.43 & 75.28 & \textbf{76.30} \\
\bottomrule
\end{tabular}
\end{minipage}\hfill
\begin{minipage}[t]{0.48\textwidth}
\centering
\textbf{(b) Sample-wise temperature optimization}\\[3pt]
\setlength{\tabcolsep}{2pt}
\begin{tabular}{lccc}
\toprule
Method & Trainable $\gamma$ & First-order & Newton \\
\midrule
TA-TTM  & $76.14_{\pm 0.17}$ & $75.88_{\pm 0.18}$ & \textbf{$76.42_{\pm 0.16}$} \\
TA-WTTM & $76.25_{\pm 0.23}$ & $75.77_{\pm 0.04}$ & \textbf{$76.57_{\pm 0.17}$} \\
\bottomrule
\end{tabular}
\end{minipage}
\label{tab:temperature_design_ablation}
\end{table*}

%% file: tabs/tabs_cifar100_wttm_ablation_short.tex
\begin{table}[t!]
\small
\setlength{\tabcolsep}{4pt}
\caption{Ablation study of the TA-WTTM temperature update on CIFAR100. We compare TA-WTTM, which uses the decoupled update, with the full-objective update on representative same- and different-architecture pairs. $\Delta$ denotes TA-WTTM minus TA-WTTM-full.}
\label{tab:cifar100_wttm_ablation}
\centering
\resizebox{\textwidth}{!}{%
\begin{tabular}{c c c c c c c}
\toprule
\multicolumn{1}{c}{Teacher}
& WRN-40-2 & ResNet110 & VGG13 & VGG13 & ResNet32$\times$4 & WRN-40-2 \\
\multicolumn{1}{c}{Student}
& WRN-16-2 & ResNet32 & VGG8 & MobileNetV2 & ShuffleNetV1 & ShuffleNetV1 \\
\midrule
TA-WTTM & $76.57_{\pm 0.17}$ & $74.04_{\pm 0.11}$ & $74.71_{\pm 0.08}$ & $68.84_{\pm 0.74}$ & $74.32_{\pm 0.23}$ & $75.72_{\pm 0.29}$ \\
TA-WTTM-full & $76.54_{\pm 0.23}$ & $73.64_{\pm 0.29}$ & $74.56_{\pm 0.24}$ & $68.74_{\pm 0.33}$ & $73.83_{\pm 0.60}$ & $75.53_{\pm 0.26}$ \\
$\Delta$ & $\textcolor{red!70!black}{+0.03}$ & $\textcolor{red!70!black}{+0.40}$ & $\textcolor{red!70!black}{+0.15}$ & $\textcolor{red!70!black}{+0.10}$ & $\textcolor{red!70!black}{+0.49}$ & $\textcolor{red!70!black}{+0.19}$ \\
\bottomrule
\end{tabular}}
\end{table}

%% file: figs/fig_hyperparameter_sensitivity.tex
\begin{figure*}[t!]
\centering
\begin{minipage}[t]{0.32\textwidth}
\centering
\includegraphics[width=\linewidth]{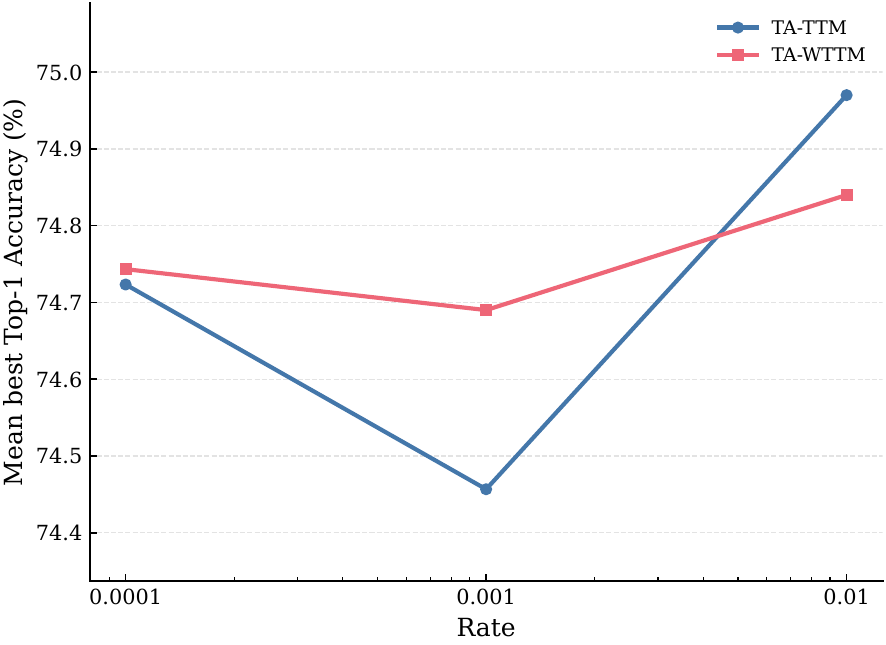}\\
(a) Update rate $\eta$
\end{minipage}\hfill
\begin{minipage}[t]{0.32\textwidth}
\centering
\includegraphics[width=\linewidth]{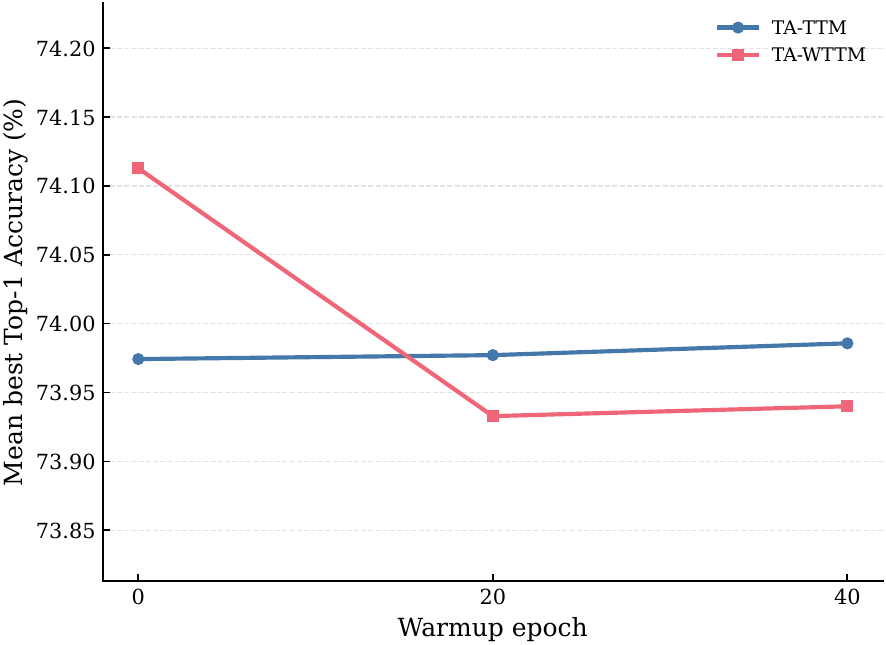}\\
(b) Warm-up length
\end{minipage}\hfill
\begin{minipage}[t]{0.32\textwidth}
\centering
\includegraphics[width=\linewidth]{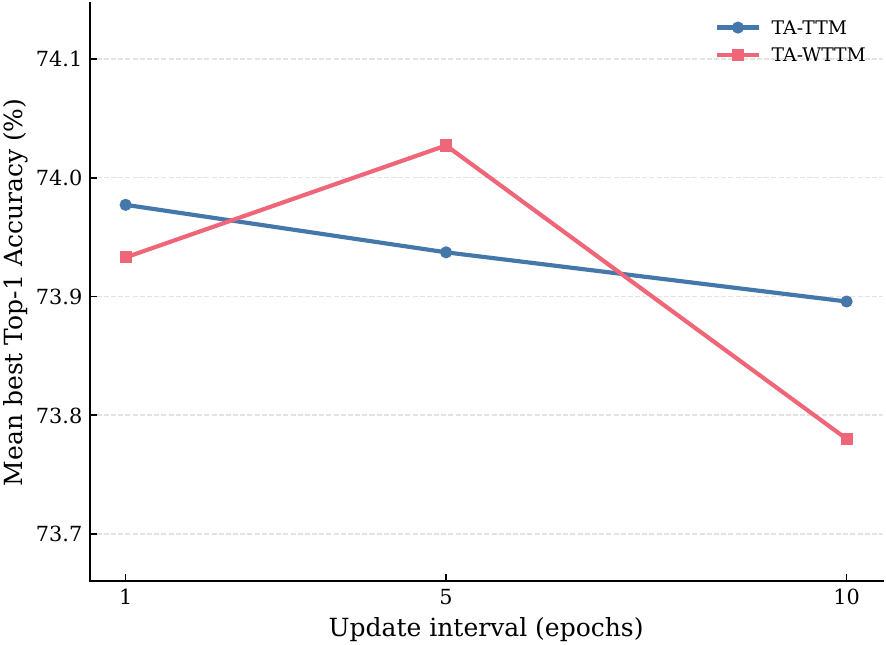}\\
(c) Update interval
\end{minipage}
\caption{Hyperparameter sensitivity of the proposed temperature adaptation. We evaluate the update rate $\eta$, warm-up length, and update interval for the sample-wise inverse temperature update.}
\label{fig:hyperparameter-sensitivity}
\end{figure*}

%% file: tabs/tabs_atd_comparison.tex
\begin{table*}[t!]
\centering
\small
\setlength{\tabcolsep}{8pt}
\caption{Comparison with ATD on CIFAR100 for WRN-40-2$\rightarrow$WRN-16-2. Results are Best Acc@1.}
\label{tab:atd_comparison}
\begin{tabular}{lccccc}
\toprule
Method & ATD & TTM & WTTM & TA-TTM & TA-WTTM \\
\midrule
Acc. & $75.83_{\pm 0.14}$ & $76.23_{\pm 0.15}$ & $76.37_{\pm 0.10}$ & \textbf{$76.42_{\pm 0.16}$} & \textbf{$76.57_{\pm 0.17}$} \\
\bottomrule
\end{tabular}
\end{table*}

%% file: figs/fig_temperature_diagnostics_cross_supp.tex
\begin{figure}[H]
\centering

\begin{minipage}[t]{0.31\textwidth}
\centering
\includegraphics[width=\linewidth]{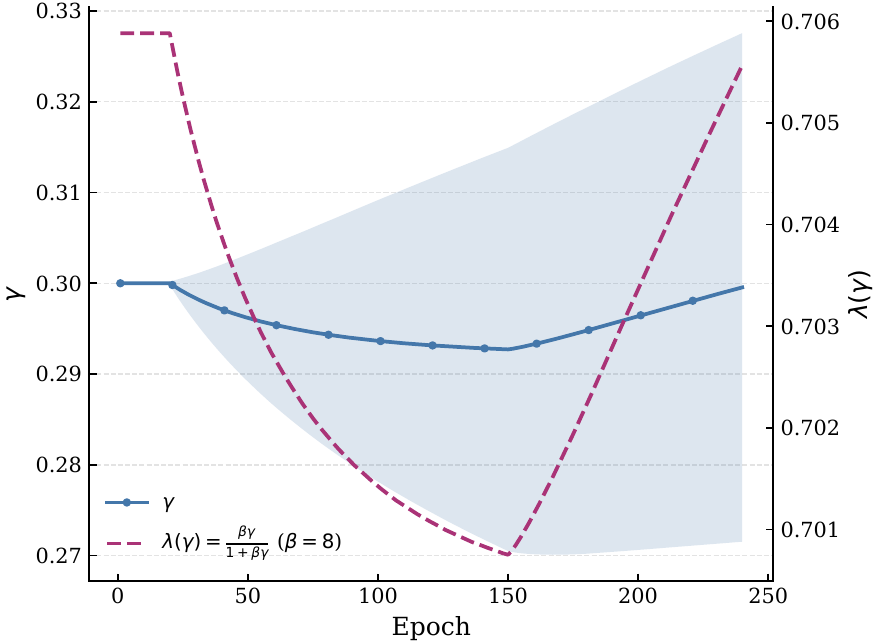}\\
(a) TA-TTM dynamics
\end{minipage}\hfill
\begin{minipage}[t]{0.31\textwidth}
\centering
\includegraphics[width=\linewidth]{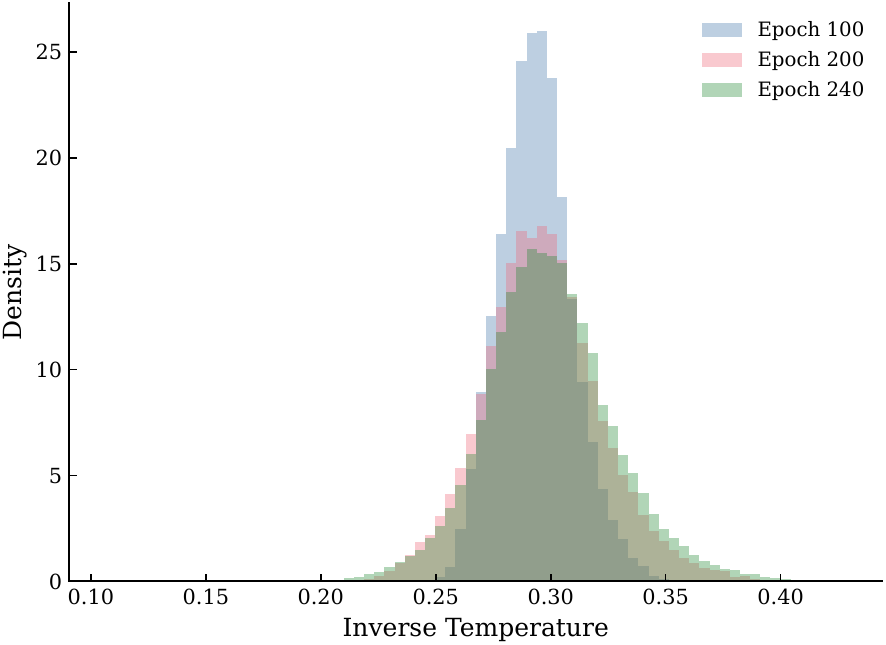}\\
(b) TA-TTM distribution
\end{minipage}\hfill
\begin{minipage}[t]{0.31\textwidth}
\centering
\includegraphics[width=\linewidth]{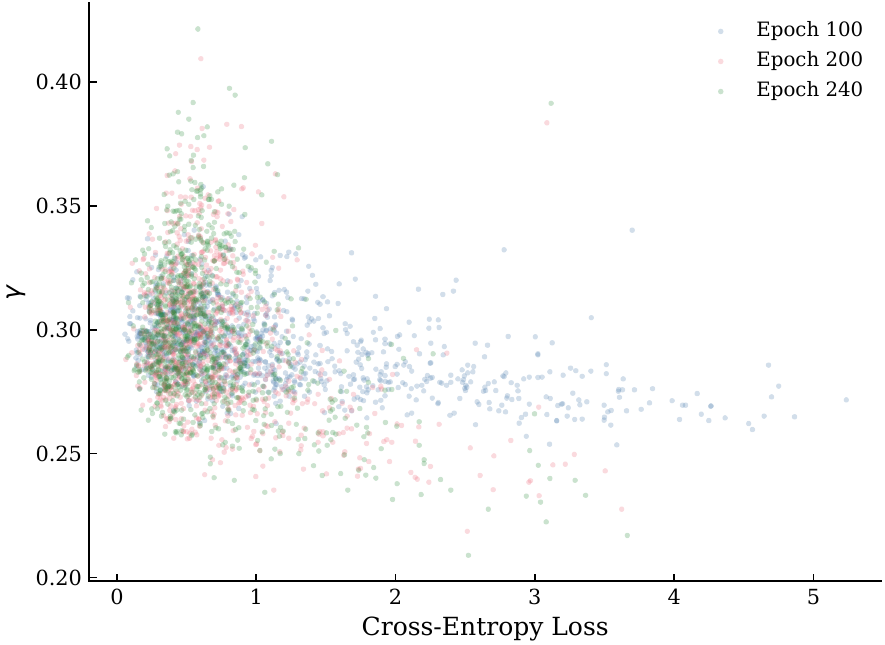}\\
(c) TA-TTM correlation
\end{minipage}

\vspace{0.35em}

\begin{minipage}[t]{0.31\textwidth}
\centering
\includegraphics[width=\linewidth]{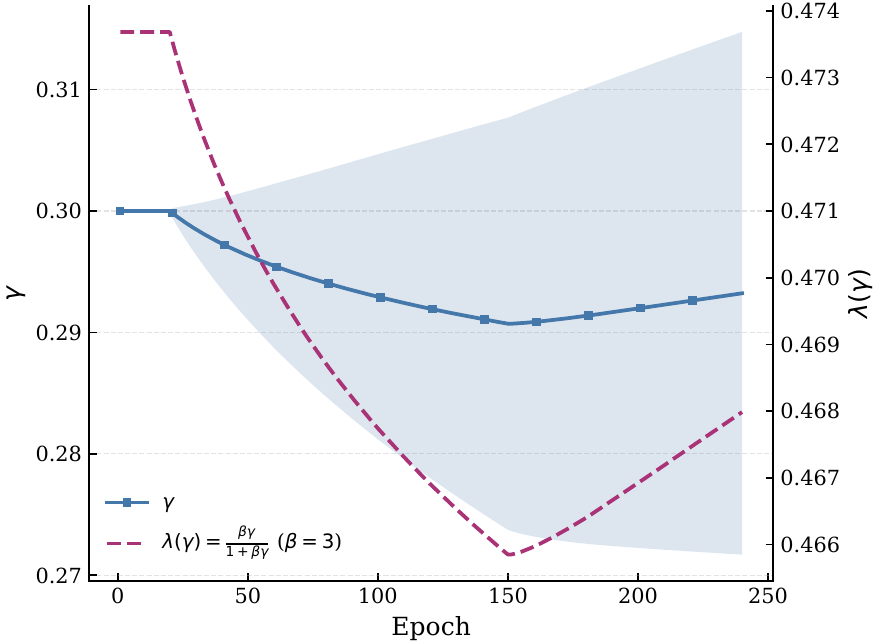}\\
(d) TA-WTTM dynamics
\end{minipage}\hfill
\begin{minipage}[t]{0.31\textwidth}
\centering
\includegraphics[width=\linewidth]{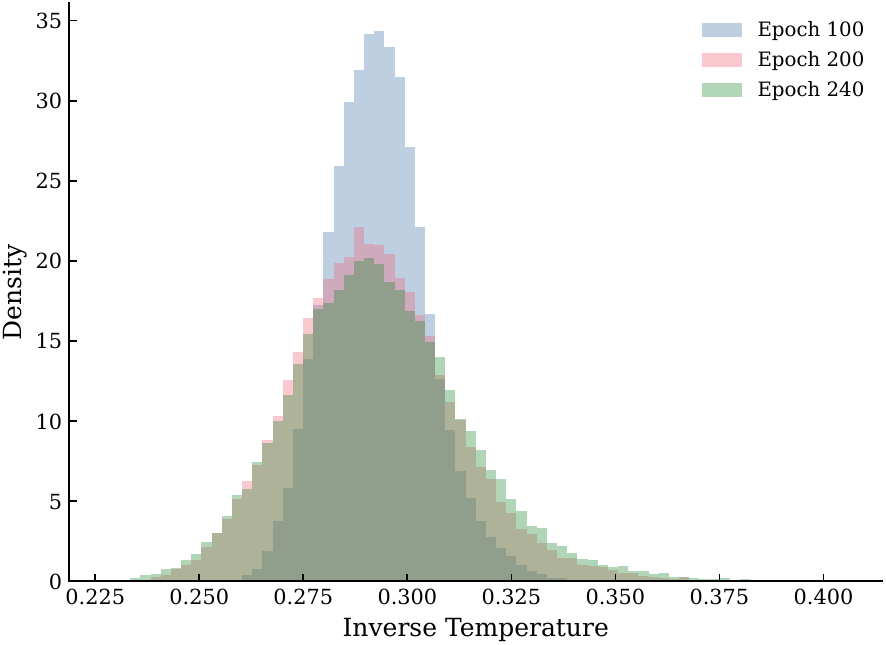}\\
(e) TA-WTTM distribution
\end{minipage}\hfill
\begin{minipage}[t]{0.31\textwidth}
\centering
\includegraphics[width=\linewidth]{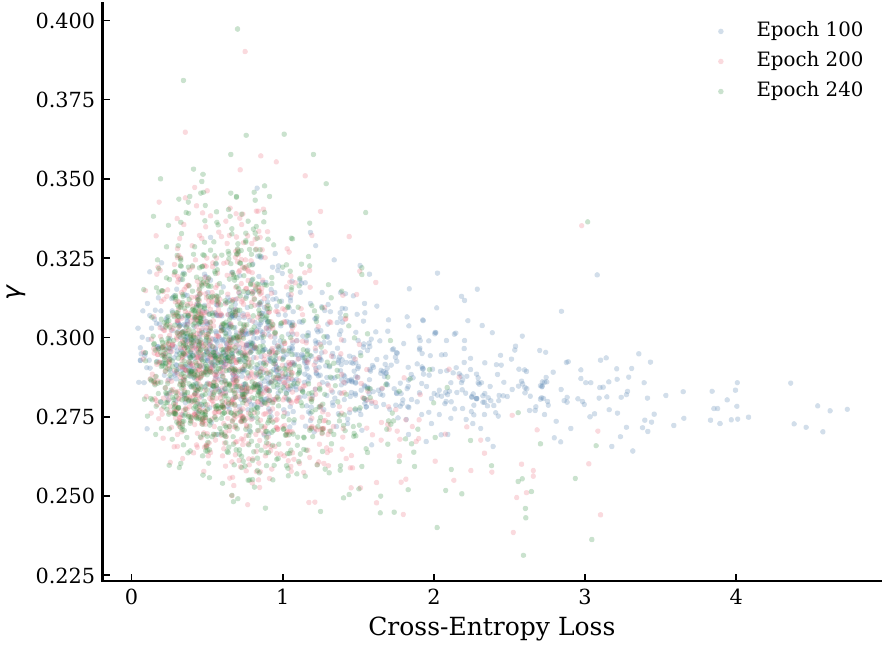}\\
(f) TA-WTTM correlation
\end{minipage}

\caption{Cross-architecture temperature diagnostics for WRN-40-2$\rightarrow$ShuffleNetV1. The first and second rows show TA-TTM and TA-WTTM, respectively; the columns show temperature dynamics, stage-wise inverse-temperature distributions, and correlations with per-sample cross-entropy loss.}
\label{fig:temperature-diagnostics-cross-supp}
\end{figure}

%% file: tabs/tabs_curvature_diagnostics.tex
\begin{table}[H]
\centering
\footnotesize
\setlength{\tabcolsep}{3pt}
\caption{Curvature diagnostics for WRN-40-2$\rightarrow$WRN-16-2, evaluated over all 50,000 training samples. The last three columns report the percentage of samples satisfying each condition.}
\label{tab:curvature_diagnostics}
\begin{tabular}{lrrrrrr}
\toprule
Method & Epoch & Mean $h_{\mathrm{update}}$ & Min. $h_{\mathrm{update}}$ & $h_{\mathrm{update}}<0$ & $|h_{\mathrm{update}}|<10^{-8}$ & $|h_{\mathrm{update}}|<10^{-3}$ \\
\midrule
TA-TTM  & 100 & 22.46 & $-35.14$ & 0.164\% & 0\% & 0\% \\
TA-TTM  & 200 & 16.79 & 1.47     & 0\%     & 0\% & 0\% \\
TA-TTM  & 240 & 16.77 & 1.83     & 0\%     & 0\% & 0\% \\
\midrule
TA-WTTM & 100 & 22.02 & $-32.06$ & 0.164\% & 0\% & 0.002\% \\
TA-WTTM & 200 & 17.03 & $-2.88$  & 0.004\% & 0\% & 0\% \\
TA-WTTM & 240 & 17.08 & 1.02     & 0\%     & 0\% & 0\% \\
\bottomrule
\end{tabular}
\end{table}

%% file: tabs/tabs_gamma_ce_correlation.tex
\begin{table}[H]
\centering
\small
\setlength{\tabcolsep}{8pt}
\caption{Correlation between the learned inverse temperature $\gamma$ and per-sample cross-entropy loss for WRN-40-2$\rightarrow$WRN-16-2, using 1,000 samples per epoch.}
\label{tab:gamma_ce_correlation}
\begin{tabular}{llccc}
\toprule
Method & Statistic & Epoch 100 & Epoch 200 & Epoch 240 \\
\midrule
TA-TTM  & Pearson  & $-0.673$ & $-0.687$ & $-0.713$ \\
        & Spearman & $-0.710$ & $-0.722$ & $-0.744$ \\
\midrule
TA-WTTM & Pearson  & $-0.673$ & $-0.645$ & $-0.683$ \\
        & Spearman & $-0.720$ & $-0.681$ & $-0.709$ \\
\bottomrule
\end{tabular}
\end{table}

%% file: tabs/tabs_hyperparams.tex
\begin{table*}[h!]
\centering
\caption{Hyperparameters used for each teacher--student pair.}
\label{tab:hyperparams}
\small
\setlength{\tabcolsep}{4pt}
\begin{tabular}{lcccc}
\toprule
Pair & Method & $\gamma_{\mathrm{init}}$ & $\beta$ & $\eta$ \\
\midrule
vgg13\_MobileNetV2 & TA-TTM & 0.2 & 16 & 0.005 \\
vgg13\_MobileNetV2 & TA-WTTM & 0.2 & 3 & 0.004 \\
\midrule
ResNet50\_MobileNetV2 & TA-TTM & 0.2 & 20 & 0.004 \\
ResNet50\_MobileNetV2 & TA-WTTM & 0.2 & 5 & 0.001 \\
\midrule
ResNet50\_vgg8 & TA-TTM & 0.1 & 70 & 0.005 \\
ResNet50\_vgg8 & TA-WTTM & 0.1 & 2 & 0.005 \\
\midrule
resnet32x4\_ShuffleNetV1 & TA-TTM & 0.2 & 12 & 0.005 \\
resnet32x4\_ShuffleNetV1 & TA-WTTM & 0.2 & 1.4 & 0.005 \\
\midrule
resnet32x4\_ShuffleNetV2 & TA-TTM & 0.4 & 40 & 0.004 \\
resnet32x4\_ShuffleNetV2 & TA-WTTM & 0.4 & 16 & 0.001 \\
\midrule
WRN-40-2\_ShuffleNetV1 & TA-TTM & 0.3 & 8 & 0.004 \\
WRN-40-2\_ShuffleNetV1 & TA-WTTM & 0.3 & 3 & 0.003 \\
\midrule
WRN-40-2\_WRN-16-2 & TA-TTM & 0.1 & 101 & 0.004 \\
WRN-40-2\_WRN-16-2 & TA-WTTM & 0.1 & 4 & 0.003 \\
\midrule
WRN-40-2\_WRN-40-1 & TA-TTM & 0.1 & 76 & 0.003 \\
WRN-40-2\_WRN-40-1 & TA-WTTM & 0.1 & 3 & 0.005 \\
\midrule
resnet56\_resnet20 & TA-TTM & 0.3 & 7 & 0.002 \\
resnet56\_resnet20 & TA-WTTM & 0.3 & 1.5 & 0.003 \\
\midrule
resnet110\_resnet20 & TA-TTM & 0.2 & 8 & 0.001 \\
resnet110\_resnet20 & TA-WTTM & 0.2 & 2 & 0.001 \\
\midrule
resnet110\_resnet32 & TA-TTM & 0.1 & 33 & 0.002 \\
resnet110\_resnet32 & TA-WTTM & 0.1 & 1.5 & 0.002 \\
\midrule
resnet32x4\_resnet8x4 & TA-TTM & 0.1 & 100 & 0.002 \\
resnet32x4\_resnet8x4 & TA-WTTM & 0.1 & 3 & 0.005 \\
\midrule
vgg13\_vgg8 & TA-TTM & 0.1 & 45 & 0.005 \\
vgg13\_vgg8 & TA-WTTM & 0.1 & 2.25 & 0.002 \\
\bottomrule
\end{tabular}
\end{table*}